\documentclass[preprint,12pt]{elsarticle}

\usepackage{amssymb}
\usepackage{amsmath}
\usepackage{graphicx} % Required for inserting images
\usepackage{algorithm2e}
\usepackage[english]{babel}
\usepackage{multirow}

\usepackage[letterpaper,top=2cm,bottom=2cm,left=3cm,right=3cm,marginparwidth=1.75cm]{geometry}

\usepackage{graphicx}

\usepackage[utf8]{inputenc}
\usepackage{color}
\usepackage{graphicx}
\usepackage{float}
\usepackage{subfigure}
\usepackage{amsfonts}
\usepackage{amsthm}

\usepackage{hyperref}

\usepackage[capitalise,nameinlink]{cleveref}

\crefname{theorem}{Theorem}{Theorems}     % 这样 \cref 输出也是大写的 "Theorem"
\crefname{lemma}{Lemma}{Lemmas}
\crefname{corollary}{Corollary}{Corollaries}
\crefname{proposition}{Proposition}{Propositions}
\crefname{definition}{Definition}{Definitions}
\crefname{remark}{Remark}{Remarks}
\crefname{example}{Example}{Examples}
\crefname{figure}{Figure}{Figures}
\crefname{table}{Table}{Tables}

\newtheorem{theorem}{Theorem}[section]

\newtheorem{lemma}[theorem]{Lemma}
\newtheorem{corollary}[theorem]{Corollary}

\newtheorem{definition}[theorem]{Definition}

\journal{Neural Networks}

\begin{document}

\begin{frontmatter}

%% Title, authors and addresses

%% use the tnoteref command within \title for footnotes;
%% use the tnotetext command for theassociated footnote;
%% use the fnref command within \author or \affiliation for footnotes;
%% use the fntext command for theassociated footnote;
%% use the corref command within \author for corresponding author footnotes;
%% use the cortext command for theassociated footnote;
%% use the ead command for the email address,
%% and the form \ead[url] for the home page:
%% \title{Title\tnoteref{label1}}
%% \tnotetext[label1]{}
%% \author{Name\corref{cor1}\fnref{label2}}
%% \ead{email address}
%% \ead[url]{home page}
%% \fntext[label2]{}
%% \cortext[cor1]{}
%% \affiliation{organization={},
%%             addressline={},
%%             city={},
%%             postcode={},
%%             state={},
%%             country={}}
%% \fntext[label3]{}

\title{Efficient Approximation to Analytic and $L^p$ functions by Height-Augmented ReLU Networks}

%% use optional labels to link authors explicitly to addresses:
\author[CUHK]{ZeYu Li}
\ead{1155184076@link.cuhk.edu.hk}
\author[CityU]{FengLei Fan}
\ead{fenglfan@cityu.edu.hk}
\author[GZNC,UIC]{TieYong Zeng\corref{Cor1}}
\ead{tieyongzeng@bnbu.edu.cn}

\cortext[cor1]{Corresponding author.}

\affiliation[CUHK]{organization={Chinese University of Hong Kong, Department of Mathematics},
            city={Hong Kong},
            postcode={999077},
            state={NT},
            country={Hong Kong SAR, China}}

\affiliation[CityU]{organization={City University of Hong Kong, Department of Data Science},
            city={Hong Kong},
            postcode={999077},
            state={Kowloon},
            country={Hong Kong SAR, China}}

\affiliation[GZNC]{organization={Guangzhou Nanfang College, School of Mathematics and Statistics},
            city={Guangzhou},
            postcode={510970},
            state={Guangdong},
            country={China}}

\affiliation[UIC]{organization={Beijing Normal Hong Kong Baptist University, Institute for Advanced Study},
            city={Zhuhai},
            postcode={519087},
            state={Guangdong},
            country={China}}

%% Abstract
\begin{abstract}
%% Text of abstract
This work addresses two fundamental limitations in neural network approximation theory. We demonstrate that a three-dimensional network architecture enables a significantly more efficient representation of sawtooth functions, which serves as the cornerstone in the approximation of analytic and $L^p$ functions. First, we establish substantially improved exponential approximation rates for several important classes of analytic functions and offer a parameter-efficient network design. Second, for the first time, we derive a quantitative and non-asymptotic approximation of high orders for general $L^p$ functions. Our techniques advance the theoretical understanding of the neural network approximation in fundamental function spaces and offer a theoretically grounded pathway for designing more parameter-efficient networks.
\end{abstract}

% %%Graphical abstract
% \begin{graphicalabstract}
% %\includegraphics{grabs}
% \end{graphicalabstract}

%%Research highlights
% \begin{highlights}
% \item \textbf{Improvement of the approximation rate for analytic functions.} By introducing the height into ReLU networks, we substantially improve the approximation
% efficiency for three types of analytic function studied in the previous work.
% \item \textbf{Quantitative and non-asymptotic approximation of high-order in $L^p$ space.} We derive a quantitative and non-asymptotic approximation error bound of any order of general $L^p$ functions for the first time.
% \end{highlights}

%% Keywords
\begin{keyword}
Deep Neural Network,  
Approximation Theory,
Analytic Functions,
$L^p$ Functions.
%% keywords here, in the form: keyword \sep keyword

%% PACS codes here, in the form: \PACS code \sep code

%% MSC codes here, in the form: \MSC code \sep code
%% or \MSC[2008] code \sep code (2000 is the default)

\end{keyword}

\end{frontmatter}

%% Add \usepackage{lineno} before \begin{document} and uncomment 
%% following line to enable line numbers
%% \linenumbers

%% main text
%%

%% Use \section commands to start a section
\section{Introduction}
\label{sec1}
%% Labels are used to cross-reference an item using \ref command.
Over the past few years, deep neural networks (NNs) have enabled remarkable advances across many important domains such as computer vision \cite{Tian2025survey,bourceanu2025foundations}, natural language processing \cite{Ilman2025role,Tucudean2024transformers}, and scientific computing \cite{Huang2025pde,jha2025theory}. This success is largely attributed to their ability to express highly complex information by composing multiple layers of affine transformations interleaved with non-linear activation functions. In this context, a substantial body of research, often referred to as neural network approximation theory \cite{Elbrächter2021deep}, explores the expressivity of deep networks, aiming to understand their strong approximation capability. Specifically, it is investigated how well a network can express a specific class of functions, \textit{e.g.}, (Hölder) continuous functions \cite{shen2020deep,voigtlaender2019approximation}, smooth functions \cite{Lu_2021}, analytic functions \cite{wang2018exponential,Beknazaryan2021Analytic,schwab2021deep}, $L^p$ functions \cite{achour2022general,CAO2009errors}. 

\begin{itemize}
    \item \textit{Continuous functions}: In \cite{shen2020deep}, a ReLU network of width $\mathcal{O}\left(W\right)$ and depth $\mathcal{O}(K)$ can approximate any $\alpha$-Hölder continuous function on $[0,1]^d$ with a nearly optimal approximation error $\mathcal{O}\left((WK)^{-2\alpha/d}\right)$. From the perspective of parameter count, \cite{voigtlaender2019approximation} showed that to reach an error $\epsilon$ in approximating a $\alpha$-Hölder continuous function on $\left[-\frac{1}{2},\frac{1}{2}\right]$, the minimal number of weights required in a ReLU NN is $\mathcal{O}\left(\epsilon^{-d/\alpha}\right)$, which is optimal for certain quantized network architectures.
    \item \textit{Smooth functions}: \cite{Lu_2021} established the approximation of arbitrary $s$-order continuously differentiable functions on $[0,1]^d$, where a ReLU NN of width $\mathcal{O}\left(W\ln W \right)$ and depth $\mathcal{O}(K\ln K)$ can reach a nearly tight error bound of $\mathcal{O}\left((WK)^{-2s/d}\right)$.
    \item \textit{Sobolev functions}: For Sobolev functions in $W^{k,\infty}\left([0,1]^d\right)$, \cite{yarotsky18Optimal} showed that given $M$ parameters, the optimal approximation error that a ReLU NN can achieve is $\mathcal{O}\left(M^{-k/d}\right)$. For the approximation in $W^{k,p}\left([0,1]^d\right)$ with $1\leq p<\infty$, \cite{siegel2024optimal} showed that the optimal approximation error a fixed-width and $L$-depth ReLU NN can reach is $\mathcal{O}\left(L^{-2k/d}\right)$.

    %\item \textit{Functions in Korobov space}:
    \item \textit{Analytic functions}: Some studies can establish an exponential approximation rate with a narrow yet deep network. For example, \cite{wang2018exponential} used the power series expansion to construct a ReLU network of fixed width and depth $\mathcal{O}\bigl(K^{2d}\bigr)$ to approximate a real analytic function on $[0,1]^d$ with an error $\mathcal{O}\bigl(\exp(-K)\bigr)$. Furthermore, similar results for analytic functions under specific holomorphic conditions were also developed. \cite{Beknazaryan2021Analytic} applied the Chebyshev approximation to the analytic functions that can be analytically continued to a Bernstein ellipse, showing that they can be approximated by networks of width $\mathcal{O}\bigl(N^{d+2}\bigr)$ and depth $\mathcal{O}\bigl(N^{2}\bigr)$ with an error $\mathcal{O}(\exp{-N})$ on $[0,1]^d$. Moreover, \cite{schwab2021deep} utilized Hermite polynomials to approximate holomorphic functions on $\mathbb{R}^d$ under Gaussian measure, achieving an $L^2$ error $\mathcal{O}\bigl(\exp(-N)\bigr)$ by networks of depth $\mathcal{O}\bigl(N^3\log^2N\bigr)$.

    \item \textit{$L^p$ functions}: \cite{CAO2009errors,Zhao2009LpError} studied the approximation of univariate $L^p$ functions using single‑hidden‑layer networks. \cite{achour2022general} provided lower and upper approximation bounds of non‑decreasing $L^p$ functions using  networks with Heaviside activation. 
\end{itemize}
From the above summary, it can be seen that i) the theory for continuous functions, smooth functions, and Sobolev functions is almost well-established, reaching nearly tight error bounds achieved by the VC-dimension analysis; ii) the approximation of analytic functions can achieve an exponential error rate on depth or width. Since they utilize polynomials as intermediate approximation, the constructed networks need at least $\mathcal{O}\left(N^2\right)$ layers to give an error of $\mathcal{O}\left(\exp{(-N)}\right)$; iii) the existing approximation results are not established directly for the general $L^p$ space. The quantitative and non-asymptotic approximation of $L^p$ functions are confined to univariate functions as targets. This is likely because the $L^p$ space lacks structural regularity that most constructive approximation proofs rely on, not like its subspaces such as the Sobolev \cite{YANG2025optimal,liu2025integral} space.
\vspace{-0.2cm}
\begin{figure}[H]
\centering  
\label{Fig_3DNet}
\includegraphics[width=1\textwidth]{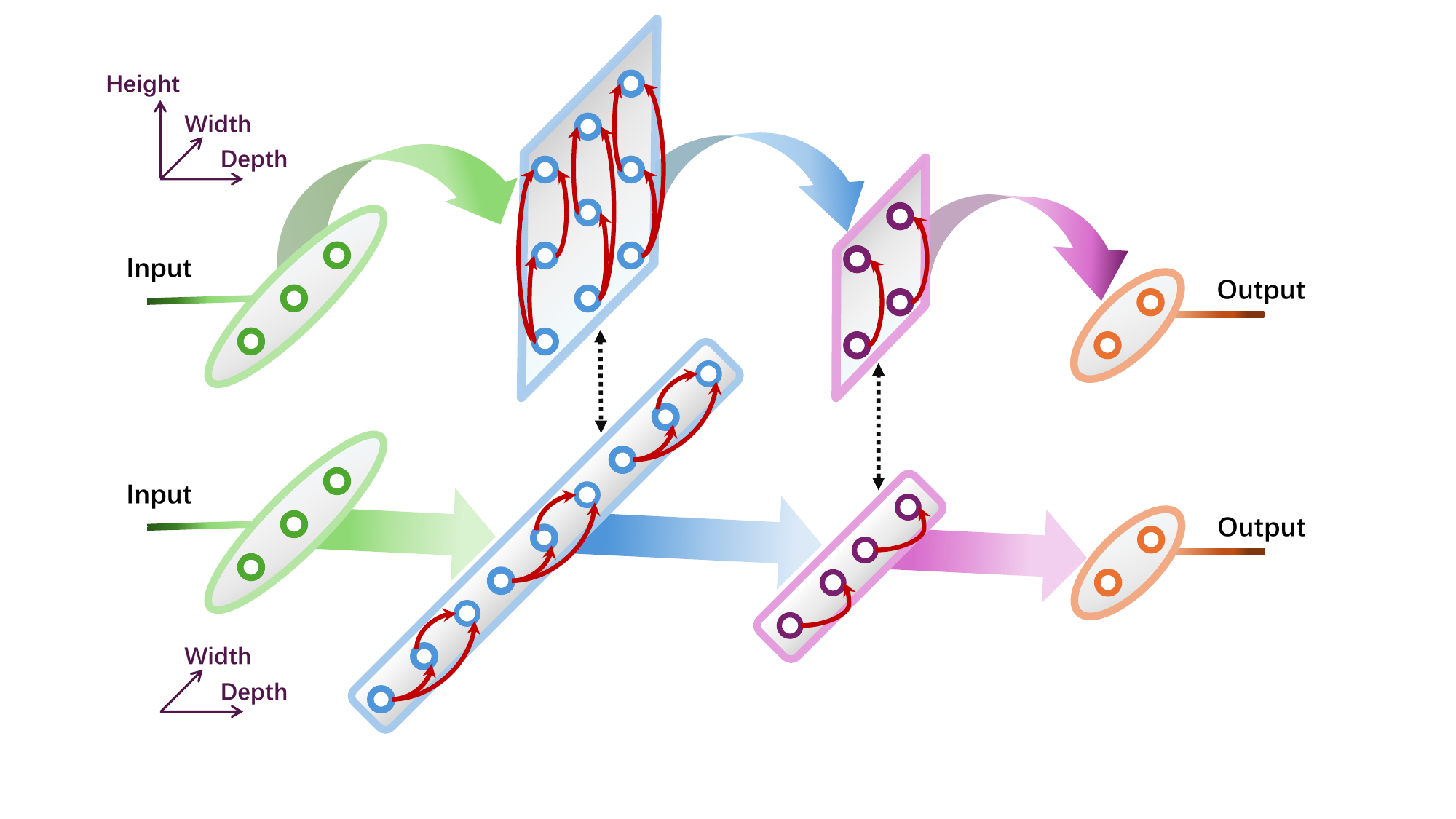}
\vspace{-1.3cm}
\caption{Adding intra-layer links creates a new hierarchy among neurons in the same layer, which induces a new dimension referred to as height. Note that a 2D network can be regarded as a 3D network with height=1.}
\label{Fig_Sparse3D}
\vspace{-0.2cm}
\end{figure}

Based on this analysis, we ask research questions regarding the universal approximation of analytic and $L^p$ functions, respectively: 1) \textit{can we improve the approximation rate for analytic functions?} and 2) \textit{can we give quantitative and non-asymptotic approximation for the general $L^p$ functions?} These two questions are theoretically important. Analytic functions represent an ideal yet practically significant class that concerns a vast array of critical phenomena in partial differential equations (PDEs) \cite{Treves2022analyticPDE}, complex analysis \cite{Bak2010complex}, and algebraic geometry \cite{Serre1956Geo}. Therefore, demonstrating such improved rates can shed light on the field of AI for science. Moreover, current empirical scaling laws often exhibit diminishing returns with increased model size. A theoretically better exponential convergence rate suggests the potential for achieving drastically higher accuracy ceilings without necessarily requiring exponentially larger models or datasets, thereby bending the scaling curve towards more favorable compute-accuracy trade-offs. The $L^p$ functions constitute foundational objects in modern analysis and its applications, forming the core of functional analysis \cite{Stein2011Functional}, harmonic analysis \cite{stein1993Harmonic}, and PDEs \cite{Evans2010pde}. Due to the well-behaved
 properties of $L^p$ spaces, \textit{i.e.}, completeness, reflexivity, and compatibility with fundamental inequalities, the $L^p$ metric is also widely used in deep learning theory. The quantitative and non-asymptotic approximation of general $L^p$ functions provides an explicit and computable error bound. Such a bound will enrich our understanding of the network's approximation ability. We find that answering these two questions leads to the same issue: how to construct a better network for representing the sawtooth function? On one hand, approximating an analytic function involves the construction of polynomials, which, according to \cite{YAROTSKY2017Error}, requires a ReLU network to first express a sawtooth function in order to construct power functions. On the other hand, we can use a trigonometric polynomial to provide a quantitative and non-asymptotic construction for general $L^p$ functions, while expressing a trigonometric polynomial also involves the sawtooth function as a core building block.

An additional topological dimension, termed height, was introduced in \cite{Fan2025Expressivity}, which is realized via intra-layer links. As seen in \cref{Fig_Sparse3D}, height-augmented three-dimensional (3D) neural networks can be realized by introducing intra-layer connections into classical two-dimensional (2D) architectures. Topologically, a 2D network with width $W$ and depth $K$ can be seen as a 3D network with width $W$, depth $K$, and height $1$. Although not specifying intra-layer links as a new dimension, \cite{AllenZhu2025Physics} applied intra-layer links in transformers as canon layers to add horizontal information flow, while \cite{zhang2022theoretically} introduced intra-layer links to improve the generalization capability in spiking networks. This hierarchical structure can significantly enhance network expressivity without substantially increasing parameter counts. At the functional level, the 3D architecture enables an exponential reduction in neuron count when representing sawtooth functions. Thus, it attains exponential approximation rates on both depth and width for polynomials, while also representing high-frequency components in trigonometric series more efficiently. Our results are summarized in Table \ref{table_summary}. In this work, we leverage this height-augmented architecture in Rectified Linear Unit (ReLU) NNs. Our key contributions are summarized as follows:

\begin{itemize}
\item \textbf{Improvement of the approximation rate for analytic functions.} 
By introducing the height in ReLU networks, we substantially improve the approximation efficiency for three common types of analytic function studied in the previous work. First, to achieve error $(1-\delta)^N$ in approximating a real analytic function with absolutely convergent power series, we improve the excessively deep network in \cite{wang2018exponential} with fixed width and depth $\mathcal{O}(N^{2d})$ to a 3D ReLU NN with width $\mathcal{O}(N^{d-1})$, depth $\mathcal{O}(N)$, and height $\mathcal{O}(N)$. Second, compared to \cite{Beknazaryan2021Analytic} that needs an $\mathcal{O}(N^2)$-depth and $\mathcal{O}(N^{d+2})$-width network for an error of $\mathcal{O}(\exp(-N))$ in approximating analytic functions on $[0,1]^d$ with holomorphic continuation to a Bernstein ellipse, we achieve the same error with only depth $\mathcal{O}(N)$,  width $\mathcal{O}(N^{d-1})$, and height $\mathcal{O}(N)$. Third, for the analytic functions in $L^2(\mathbb{R}^d,\gamma_d)$, which are holomorphic in a complex strip, \cite{schwab2021deep} utilizes a network of depth $\mathcal{O}\left(N \log^2 N\right)$ to reach an error $\mathcal{O}\left(\exp\left(-N^{\frac{1}{3}}\right)\right)$, while we only need $\mathcal{O}(N)$ layers and reduce the error to $\mathcal{O}\left(\exp\left(-N^{\frac{1}{2}}\right)\right)$.

\item \textbf{Quantitative and non-asymptotic approximation of high-order in $L^p$ space.} For any $r\in\mathbb{N}^+$ and $1\leq p\leq \infty$, we derive a quantitative and non-asymptotic approximation error bound of order $r$ for general $L^p$ functions. To the best of our knowledge, this is the first such approximation for general $L^p$ functions.
\end{itemize}

\begin{table}[!h]
\centering\footnotesize
\label{Result_3Dvs2D}
\renewcommand{\arraystretch}{1.2}
\setlength{\tabcolsep}{7pt}
\scalebox{0.9}{
\begin{tabular}{c|l|l|l|l|l}
\hline
Target Function & Result & Width & Depth & Height & Error \\
\hline
\multirow{2}{*}{\shortstack{Polynomial \\ on [0,1]}} 
& Proposition \ref{approx_poly} & $\mathcal{O}(1)$ & $\mathcal{O}(1)$ & $\mathcal{O}(N)$ & $\mathcal{O}\left(2^{-N}\right)$ \\
\cline{2-6}
& Proposition III.5, \cite{Elbrächter2021deep} & $\mathcal{O}(1)$ & $\mathcal{O}(N)$ & 1 & $\mathcal{O}(\exp(-N))$ \\
\hline

\multirow{2}{*}{\shortstack{Analytic on\\ $[0,1-\delta]^d$}} 
& Theorem \ref{real_analytic} & $\mathcal{O}(N^{d-1})$ & $\mathcal{O}(N)$ & $\mathcal{O}(N)$ & $\mathcal{O}\left((1-\delta)^{N}\right)$ \\
\cline{2-6}
& Lemma III.7, \cite{wang2018exponential} & $\mathcal{O}(1)$ & $\mathcal{O}(N^{2d})$ & 1 & $\mathcal{O}\left((1-\delta)^{N}\right)$ \\
\hline

\multirow{2}{*}{\shortstack{Analytic on $[0,1]^d$,\\ holomorphic in an ellipse}} 
& Theorem \ref{holomorphic_ellipse} & $\mathcal{O}(N^{d-1})$ & $\mathcal{O}(N)$ & $\mathcal{O}(N)$ & $\mathcal{O}\left(\rho^{-N}\right)$ \\
\cline{2-6}
& Theorem2.2, \cite{Beknazaryan2021Analytic} & $\mathcal{O}(N^{d+2})$ & $\mathcal{O}(N^2)$ & 1 & $\mathcal{O}\left(2^{-N}\right)$ \\
\hline

\multirow{2}{*}{\shortstack{Analytic on $L^2\left(\mathbb{R}^d,\gamma_d\right)$, \\holomorphic in a strip}} 
& Theorem \ref{holomorphic_R^d} & $\mathcal{O}(N^{d+1})$ & $\mathcal{O}(N)$ & $\mathcal{O}(N)$ & $\mathcal{O}\left(\exp\left(-N^{\frac{1}{2}}\right)\right)$ \\
\cline{2-6}
& Theorem4.7, \cite{schwab2021deep} & - & $\mathcal{O}\left(N \log^2 {N}\right)$ & 1 & $\mathcal{O}\left(\exp\left(-N^{\frac{1}{3}}\right)\right)$ \\
\hline

% \shortstack{$L^p$ functions\\ on $[-1,1]^d$} & Theorem \ref{approx_Lp} & $\mathcal{O}\left(N^d\right)$ &$\mathcal{O}(L)$ &$\mathcal{O}\left(L+\log_2 N\right)$ &\shortstack{$\mathcal{O}\left(\omega_r^d\left(f,N^{-1}\right)\right.$\\+$\left. N^d 2^{-L}\right)$}\\
\multirow{2}{*}{\shortstack{$L^p$ functions\\ on $[-1,1]^d$}} & \multirow{2}{*}{Theorem \ref{approx_Lp}} & \multirow{2}{*}{$\mathcal{O}\left(N_1^d\right)$} & \multirow{2}{*}{$\mathcal{O}(N_2)$} & \multirow{2}{*}{$\mathcal{O}\left(N_2+\log_2 N_1\right)$} & $\mathcal{O}\left(\omega_r^d\left(f,N_1^{-1}\right)_p\right.$ \\
& & & & & +$\left. N^d 2^{-N_2}\right)$ \\

\hline
\end{tabular}}
\label{table_summary}
\caption{Comparisons between our work and the existing results. $\gamma_d$ is the $d$-dimensional Gaussian measure. $\omega_d^r(f,\cdot)_p$ is the $L^p$ modulus of smoothness of order $r$ in $\mathbb{R}^d$.  Note that a 3D network of width $W$, depth $K$, and height $H$ is topologically equivalent to a 2D intra-linked network of width $W\times H$ and depth $K$.}
\end{table}

\section{Notation and Definition}
\label{sec_notation}

First, we summarize the basic notations used in this work.

\begin{itemize}
    \item Vectors and matrices are bold-faced. For example, for any $\boldsymbol{x}\in\mathbb{R}^d$, $\boldsymbol{x}=(x_1,\ldots,x_d)$. In particular, $\boldsymbol{j}$ is a multi-index $\boldsymbol{j}=(j_1,\ldots,j_d)\in\mathbb{N}^d$. 
    \item For vectors $\boldsymbol{x},\boldsymbol{y}\in\mathbb{R}^d$ and $a,b\in\mathbb{R}$, the inequality $a<\boldsymbol{x}<b$ means that $a<x_k<b$ for all $1\leq k\leq d$. Similarly, the inequality $\boldsymbol{x}<\boldsymbol{y}$ means that $x_k<y_k$ for all $1\leq k\leq d$.

    \item  For $\boldsymbol{x}\in\mathbb{R}^d$ and $\boldsymbol{j}\in\mathbb{N}^d$, we use the notation $\boldsymbol{x}^{\boldsymbol{j}}$ to denote the monomial $\boldsymbol{x}^{\boldsymbol{j}}=x_1^{j_1}\cdots x_d^{j_d}$.
    
    \item $\Pi_{n}^d$ is the set of all $2$-periodic trigonometric polynomials in $\mathbb{R}^d$ of degree at most $n$:
 \begin{equation*}
     \Pi_{n}^d=\left\{\sum_{0\leq \boldsymbol{j} \leq n}a_{\boldsymbol{j}}\cos{\left(j_1\pi x_1-\frac{\eta_1\pi}{2}\right)}\cdots \cos{\left(j_d\pi x_d-\frac{\eta_d\pi}{2}\right)}:a_{\boldsymbol{j}}\in\mathbb{R}\right\},
 \end{equation*}
 where $\eta_k=0$ corresponds to a cosine term and $\eta_k=1$ corresponds to a sine term (since $\cos\left(\theta-\frac{\pi}{2}\right)=\sin\theta$).

    \item For $d\in\mathbb{N}$, we denote the standard Gaussian measure on $\mathbb{R}^d$ as $\gamma_d$, whose density w.r.t. the Lebesgue measure on $\mathbb{R}^d$ is given by
\begin{equation}
\label{eq_Gaussian}
    \frac{1}{(2\pi)^{d/2}}\exp{\left(-\frac{\left\|\boldsymbol{x}\right\|_2^2}{2}\right)},\quad\forall \boldsymbol{x}\in\mathbb{R}^d.
\end{equation}

    \item  We denote the $n$-th probabilists' Hermite polynomial normalized in $L^2(\mathbb{R},\gamma_1)$ as
\begin{equation}
\label{eq_Hermite}
    \Xi_n(x)=\frac{(-1)^n}{\sqrt{n!}}e^{x^2/2}\frac{d^n}{dx^n}e^{-x^2/2}
\end{equation}
with the usual convention $0!=1$. Note that $\left\{ \Xi_n \right\}_{n\in\mathbb{N}_0}$ is an orthonormal basis of $L^2(\mathbb{R},\gamma_1)$, \textit{i.e.},
\begin{equation*}
\langle \Xi_n,\Xi_m\rangle=\int_\mathbb{R}\Xi_n(x)\Xi_m(x)\ d\gamma_1(x)=\delta_{n,m}.
\end{equation*}

For a multi-index $\boldsymbol{\nu}\in \mathbb{N}^d$ and $x\in\mathbb{R}^d$, we use $\Xi_{\boldsymbol{\nu}}(\boldsymbol{x})$ to denote $\Xi_{\boldsymbol{\nu}}(\boldsymbol{x})=\prod_{j=1}^d \Xi_{\nu_j}(x_j)$.
\end{itemize}

Then, we provide the formal definitions of function classes of interest, including three common types of analytic functions: real analytic functions on an interval $[0,1-\delta]^d$, and analytic functions on $[0,1]^d$ with holomorphic extensions to a complex ellipse, and analytic functions in $L^2\left(\mathbb{R}^d,\gamma_d\right)$ with holomorphic extensions to a complex strip. The reason why they are special is because they can be approximated by polynomials series with an error that decays exponentially in the polynomial degree

\begin{definition}[Real analytic functions]
    A function $f:[0,1-\delta]^d$ is said to be real analytic, if it has an absolutely convergent power series $\sum_{\boldsymbol{j}\geq 0}a_{\boldsymbol{j}}\boldsymbol{x}^{\boldsymbol{j}}$.
\end{definition}

\begin{definition}[Holomorphic functions]
    Given a function $f$ defined on a subset of $\mathbb{R}^d$, we say $f$ is analytic (or holomorphic) on a complex region if $f$ can be analytically continued to that region.
\end{definition}
Especially, we are interested in the following two types of complex regions that are  widely studied in literature:
\begin{itemize}
    \item Complex ellipse \cite{Trefethen2016Multivariate,Beknazaryan2021Analytic}:
    \begin{equation}
    \label{eq_ellipse}
        E_{d,\rho}=\left\{\boldsymbol{z}\in\mathbb{C}^d:\left|\boldsymbol{z}\right|+\left|\left|\boldsymbol{z}\right|-d\right|\leq d+2h^2\right\},
    \end{equation}
    for some $\rho>0$, where $h=(\rho-\rho^{-1})/2$.
    \item Complex strip \cite{Wang2023Convergence,schwab2021deep}:
    \begin{equation}
    \label{eq_strip}
        S_{\boldsymbol{\tau}}=\otimes_{j=1}^d\left\{z_j=x_j+\operatorname{i}y_j\in\mathbb{C}:|y_j|<\tau_j\right\}\subseteq\mathbb{C}^d,
    \end{equation}
    for some $\boldsymbol{\tau}>0$.
\end{itemize}
For approximating holomorphic functions on $\mathbb{R}^d$, we need to use the Hermite polynomials and the Gaussian measure.

\begin{definition}
Let $\Omega \subseteq \mathbb{R}^d$ be a measurable set. For $1 \leq p \leq \infty$, we define the  $L^p(\Omega)$ space as the set of all measurable functions $f: \Omega \to \mathbb{R}$ such that the $L^p$ norm
\begin{equation*}
\|f\|_{L^p(\Omega)} := 
\begin{cases}
\displaystyle \left( \int_{\Omega} |f(\boldsymbol{x})|^p \, d\boldsymbol{x} \right)^{1/p}, & 1 \leq p < \infty, \\
\displaystyle \operatorname*{ess\,sup}_{\boldsymbol{x} \in \Omega} |f(\boldsymbol{x})|, & p = \infty,
\end{cases}
\end{equation*}
is finite. Functions that are equal almost everywhere are identified as the same element in $L^p(\Omega)$.
\end{definition}

For $f\in L^p\left([-T,T]^d\right)$ with $T>0$ and $p\geq 1$, we characterize the approximation behavior by the periodic modulus of smoothness \cite{Aliev2023Jackson}. For this purpose, we extend $f$ periodically to the entire space $\mathbb{R}^d$ by defining
 \begin{equation}
     f(\boldsymbol{x}+2\boldsymbol{k}T)=f(\boldsymbol{x}), \quad \forall \boldsymbol{x}\in [-T,T]^d,\boldsymbol{k}\in \mathbb{Z}^d.
 \end{equation}
 
\begin{definition}[$L^p$ modulus of smoothness]
    The $L^p$ modulus of smoothness of order $r$ is given by
\begin{equation}
\label{eq_modulus}
    \omega_{r}(f, t)_{p}=\sup _{1 \leq i \leq d, 0<h \leq t}\left(\int_{[-T,T]^d}\left|\Delta_{i, h}^{r} f(\boldsymbol{x})\right|^{p} d \boldsymbol{x}\right)^{1/p},
\end{equation}
for $1\leq p<\infty$, and
\begin{equation}
    \omega_{r}(f, t)_{\infty}=\sup _{1 \leq i \leq d, 0<h \leq t}\sup_{\boldsymbol{x}\in [-T,T]^d}\left|\Delta_{i, h}^{r} f(\boldsymbol{x})\right|,
\end{equation}
where $\Delta_{i, h}^{r} f$ is the $r$-th difference in the direction of the $i$-th coordinate. Formally, we have
\begin{equation}
\label{eq_diff}
    \Delta_{i,h}^{r} f(\boldsymbol{x})=\sum_{k=0}^{r}(-1)^{r-k}\binom{r}{k} f\left(\boldsymbol{x}+kh \boldsymbol{e}_i\right).
\end{equation}
\end{definition}

In this work, we apply the ReLU activation, \textit{i.e.}, for any $\boldsymbol{x}\in\mathbb{R}^d$,
\begin{equation*}
    \sigma(\boldsymbol{x})=\left( \sigma(x_1),\ldots,\sigma(x_d) \right)
    =\left( \max\{x_1,0\},\ldots,\max\{x_d,0\} \right).
\end{equation*}

\begin{definition}[3D networks \cite{Fan2025Expressivity}]
\label{def_3DNet}
    For a 3D $\mathbb{R}^{w_{0}} \rightarrow \mathbb{R}$ ReLU NN with $K$ hidden layers, where in the $k$-th layer, there are $H_k$ floors with $(w_{k1},\cdots,w_{kH_k})$ neurons at each floor. Let $\mathbf{G}_{h_1,h_2}^k \in \mathbb{R}^{w_{kh_1}\times w_{kh_2}}$, $h_2>h_1$ denote the connecting operations between the $h_1$-th and $h_2$-th floors within the $k$-th hidden layer. If $\mathbf{G}^k_{h_1,h_2}(n_1,n_2)\neq 0$, it means that the output of the $n_1$-th neuron in the $h_1$-th floor is fed into the $n_2$-th neuron in the $h_2$-th floor, and multiplied by a coefficient $\mathbf{G}^k_{h_1,h_2}(n_1,n_2)$; otherwise, the output of the $n_1$-th neuron is not. $\mathbf{G}^k_{h_1,h_2 \leq h_1}=0$ by default, since no loops are allowed in a network.We use 
    
$$    \tilde{\mathbf{f}}_{0}=\mathbf{x} \in \mathbb{R}^{w_{0}} \text{ and } \tilde{\mathbf{f}}_{k}=\left[\tilde{\mathbf{f}}_{k}^{(1)}, \ldots, \tilde{\mathbf{f}}_{k}^{\left(w_{k1}\right)},\ldots, \tilde{\mathbf{f}}_{k}^{(\sum_{t=1}^{H_k-1}w_{kt}+1)}, \ldots, \tilde{\mathbf{f}}_{k}^{\left(\sum_{t=1}^{H_k}w_{kt}\right)}\right] \in \mathbb{R}^{\sum_{t=1}^{H_k}w_{kt}}$$
to denote the input and the outputs of the $k$-th layer, respectively. The $j$-th pre-activation in the $k$-th layer and the output of the network are computed as the following:  
$$g_{k}^{(j)}=\left\langle\mathbf{a}_{i}^{(j)}, \tilde{\mathbf{f}}_{k-1}\right\rangle+b_{i}^{(j)} \text{ and }
\tilde{f}_{k}^{(j)}=\sigma\left(g_{k}^{(j)}+\sum_{p<j}\mathbf{G}_{h(p),h(j)}^k(n(p),n(j))\tilde{f}_{k}^{(p)}\right)$$  
for each $j$, where $h(l)$ is a function mapping the $l$-th neuron into their floors, and $n(l)$ is a function mapping the $l$-th neuron into the order in its floor.
\end{definition}

For an $\mathbb{R}^{w_{0}} \rightarrow \mathbb{R}$ 3D ReLU NN with neurons $[(w_{11},..,w_{1H_1}), \ldots, (w_{K1},..,w_{KH_K})]$ in $K$ hidden layers, respectively, the depth of a network is $K+1$, the width is $W=\max_k\max\{w_{k1},\ldots, w_{kH_k}\}$, and the height is $H=\max\{H_1,\ldots,H_K\}$. We denote the set of all 3D ReLU NNs of width $W$, depth $K$ and height $H$ by $\mathcal{N}_{W,K,H}$. In this work, we use a sparse 3D network, where $G^k$ is a sparse matrix.

\begin{definition}
  The sawtooth function $g_s$ on $[0,1]$ with $2^{s-1}$ “sawteeth" is  formally defined by
\begin{equation}
\label{eq_sawtooth}
    g_s=\begin{cases} -2\left|x-\frac{1}{2}\right|+1,  & \text{if }s=1, \\  
    g_1\circ g_{s-1} ,  & \text{if }s>1.
    \end{cases}
\end{equation}
\end{definition}

\section{Approximation of Analytic Functions}
In this section, we establish constructive approximations of three types of analytic functions of interest. First, in Section \ref{sec_poly}, we utilize 3D ReLU NNs to approximate arbitrary univariate polynomials and $C^{\infty}$ functions based on the sawtooth function. Then in Section \ref{sec_analytic_cube}, we leverage the construction in Section \ref{sec_poly} for polynomials and establish an exponential error rate for two types of analytic functions: real analytic functions with absolutely convergent power series on $[0,1-\delta]^d$ and analytic functions in $[0,1]^d$ with holomorphic extension to a specific complex ellipse. Because the approaches we deal with these two types of analytic functions are similar, we put them in the same place. Finally, in Section \ref{sec_analytic}, we introduce the Hermite polynomials and Gaussian measure to establish the global approximation in $\mathbb{R}^d$ of any analytic functions that can be analytically continued to a specific complex strip.

\subsection{Approximation to Polynomials and $C^{\infty}$ Functions}
\label{sec_poly}
\begin{lemma}
    For any $H>0$, there is a 3D ReLU network $\mathcal{N}_{2,1,H}$, whose function $f_H: [0,1]\to [0,1]$ satisfies 
\begin{equation}
    \sup_{x\in[0,1]}\left|f_H(x)-x^2\right|\leq 2^{-2(H+1)}.
\end{equation}
\label{prop_square}
\end{lemma}

\begin{proof}
    Let $g_s:[0,1]\to [0,1]$ be the sawtooth function with $2^{s-1}$ “sawteeth" on $[0,1]$. Based on Eq. \eqref{eq_sawtooth}, $g_s$ can be reformulated as
    \begin{equation}
    g_s(x)=\begin{cases} 2x-4\sigma\left(x-\frac{1}{2} \right),  & \text{if }s=1,\\
        -4\sigma\left(-g_{s-1}(x)+\frac{1}{2} \right)-\frac{1}{2}g_{s-1}(x)+2,  & \text{if }s>1,
    \end{cases}
\end{equation}
    which can be implemented by a  ReLU network $\mathcal{N}_{1,1,s}$, where $-g_{s-1}+\frac{1}{2}$ serves as the pre-activation in the $s$-th floor for $s>1$.

Let $f_s:[0,1]\to [0,1]$ be the piecewise linear interpolation of $y=x^2$ on $x=\frac{l}{2^s},l=0,1,\ldots,2^s$. Then for any $s\in\mathbb{N}^{+},x\in[0,1]$, 
\begin{equation}
    |f_s(x)-x^2|\leq 2^{-2(s+1)}.
\end{equation}
Note that for any $s\geq 2$, $f_{s-1}-f_{s}=g_s/2^{2s}$. Thus, we have
\begin{equation}
    f_{H}(x)=f_{0}(x)+\sum_{j=1}^{H}\left(f_{j}-f_{j-1} \right)=x-\sum_{j=1}^{H}\frac{g_j(x)}{2^{2j}},
\end{equation}
which can be implemented by a 3D ReLU network $\mathcal{N}_{2,1,H}$, as shown in \cref{Fig_SquareNet}.
\end{proof}

\vspace{-0.8cm}
\begin{figure}[H]
\centering  
\label{Fig_SquareNet}
\includegraphics[width=1\textwidth]{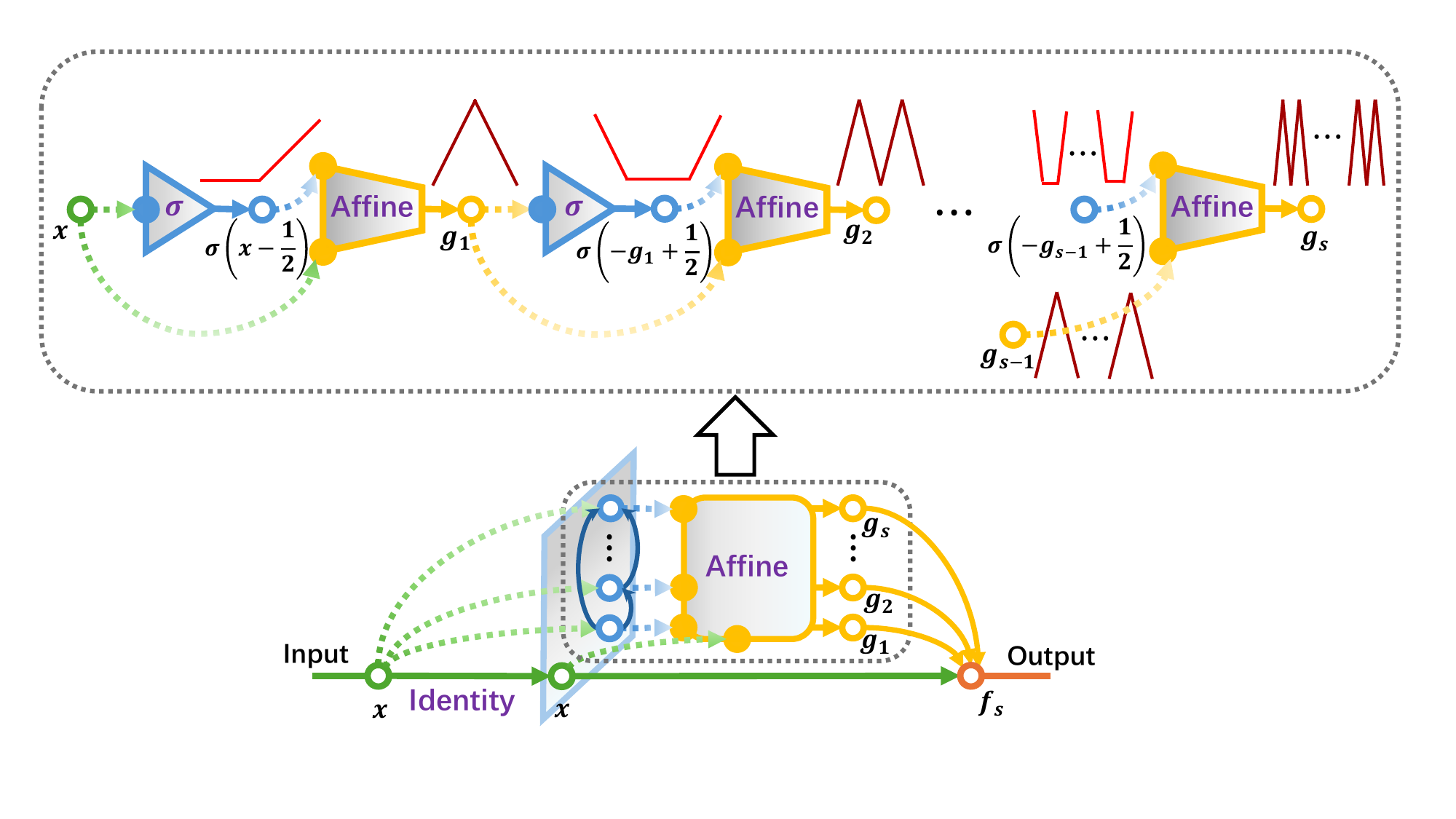}
\vspace{-1.5cm}
\caption{Implementing $f_H$ by 3D networks in \cref{prop_square}.}
\label{Fig_Sqr_Prod}
\end{figure}

% \begin{proposition}
%     For any $H>0$, there is an 3D ReLU NN $f:[0,1]\to [0,1]$ in $\mathcal{N}_{H,D,H}$ such that
% \begin{equation}
%     \left|f(x)-x^2\right|\leq 2^{-2(H+1)}.
% \end{equation}
% \end{proposition}

\begin{lemma}
\label{bi_prod}
    For any $H>0$, there is a 3D ReLU $\mathcal{N}_{6,1,H}$, whose function $\widehat{\times}:[0,1]^2\to [0,1]$ satisfies
\begin{equation}
    \sup_{x,y\in [0,1]}\left|\widehat{\times}(x,y)-xy\right|\leq 6\cdot2^{-2(H+1)}.
\end{equation}
\end{lemma}

\begin{proof}
    Since for any $x,y\in[0,1]$, we have
\begin{equation}
    xy=2\left(\frac{x+y}{2}\right)^2-2\left(\frac{x}{2}\right)^2-2\left(\frac{y}{2}\right)^2.
\end{equation}
Construct $f_{H}$ as in \cref{prop_square} and let
\begin{equation}
\label{eq_biProd}
    \widehat{\times}(x,y)=2f_H\left(\frac{x+y}{2}\right)-2f_H\left(\frac{x}{2}\right)-2f_H\left(\frac{y}{2}\right),
\end{equation}
then we have
\begin{equation}
\label{eq_BiProdErr}
\begin{split}
        \left|\widehat{\times}(x,y)-xy\right|\leq & 2\left|f_H\left(\frac{x+y}{2}\right)-\left(\frac{x+y}{2}\right)^2 \right|+2\left|f_H\left(\frac{x}{2}\right)-\left(\frac{x}{2}\right)^2 \right|+2\left|f_H\left(\frac{y}{2}\right)-\left(\frac{y}{2}\right)^2 \right|\\
        \leq & 6 \cdot2^{-2(H+1)}.
\end{split}
\end{equation}

Since we need to recursively use $\widehat{\times}(x,y)$, it remains to verify that $\widehat{\times}(x,y)$ also takes values in $[0,1]$ for all $x,y\in[0,1]$ to avoid unbounded growth. Observe that $\widehat{\times}(x,y)$ is linear in $[0,1]^2\setminus \left( \mathfrak{D}1\cup \mathfrak{D}2 \cup \mathfrak{D}3 \right)$, where 
\begin{displaymath}
    \mathfrak{D}1=\left\{ \left( \frac{2j}{2^{H}},y\right):j=0,1,\ldots,2^{H-1},y\in[0,1] \right\},
\end{displaymath}

\begin{displaymath}
    \mathfrak{D}2=\left\{\left( x, \frac{2k}{2^{H}}\right):x\in[0,1],k=0,1,\ldots,2^{H-1} \right\},
\end{displaymath}
and
\begin{displaymath}
    \mathfrak{D}3=\left\{ ( x,y)\in[0,1]^2:x+y=\frac{2l}{2^{H}},l=0,1,\ldots,2^{H-1} \right\}.
\end{displaymath}
Moreover, $\widehat{\times}(x,y)=xy$ on $\mathfrak{D}1\cap \mathfrak{D}2 \cap \mathfrak{D}3=\left\{\left( \frac{2j}{2^{H}},\frac{2k}{2^{H}}\right),j,k=0,1,\ldots,2^{H-1}\right\}$. Therefore, $\widehat{\times}(x,y)$ is the piece-ise linear interpolation of $xy$ on $\mathfrak{D}1\cap \mathfrak{D}2 \cap \mathfrak{D}3$. Since the values on all interpolation points are in $[0,1]$, their linear interpolation $\widehat{\times}(\cdot,\cdot)$ is also in $[0,1]$.
\end{proof}

\begin{lemma}
\label{approx_poly}
    For any $H,n\in\mathbb{N}^+$ and polynomial $p(x)=\sum_{k=0}^{n}a_kx^k$, there is a 3D ReLU NN $\mathcal{N}_{8,n-1,H}$, whose function $\Phi:[0,1]\to \mathbb{R}$ fulfills
\begin{equation}
    \sup_{x\in [0,1]}\left|\Phi(x)-p(x)\right|\leq \max_{1\leq k\leq n}|a_k| 3n^2 2^{-2(H+1)}.
\end{equation}
\end{lemma}

\begin{proof}
    We recursively define for $k\geq 1$
\begin{equation}
    h_k\left(x\right)=\begin{cases}
        x, & \text{if } k=1,\\
        \widehat{\times}\left(x,h_{k-1}(x)\right), & \text{if } k\geq 2.
    \end{cases}
\end{equation}
Then for $x\in [0,1]$ and $k\geq 2 $, we have
\begin{equation}
\begin{split}
    \left|h_k(x)-x^k\right|&\leq \left|\widehat{\times}\left(x,h_{k-1}(x)\right)- xh_{k-1}(x)\right|+
    \left|xh_{k-1}(x)-x^k\right|\\
    \nonumber&\leq 6\cdot2^{-2(H+1)}+
    \left|h_{k-1}(x)-x^{k-1}\right|\\
    \nonumber &\leq \cdots \leq 6(k-1)\cdot2^{-2(H+1)}.
\end{split}
\end{equation}
Let $\Phi(x)=\sum_{k=1}^{n}a_k h_k(x)+a_0$, then we have for any $x\in [0,1]$,
\begin{equation}
    \left|\Phi(x)-p(x)\right|\leq \sum_{k=1}^n \left|a_k\right| 6 (k-1)\cdot 2^{-2(H+1)} \leq \max_{k}|a_k| 3n^2 2^{-2(H+1)}.
\end{equation}
The network implementing $\Phi(x)$ is illustrated in \cref{Fig_UniVarPoly}, which is of width $8$, depth $n-1$, and height $H$.

\end{proof}

\begin{figure}[H] 
\centering 
\includegraphics[width=1\textwidth]{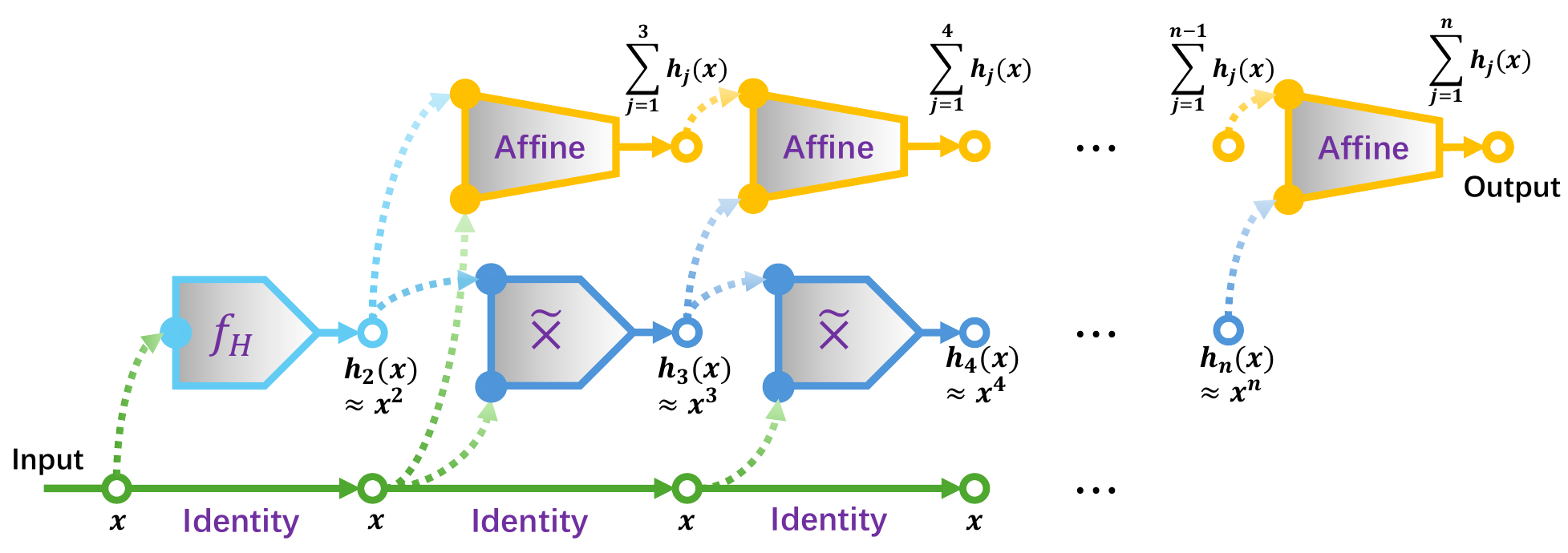} 
\caption{Illustration of the construction in \cref{approx_poly}} 
\label{Fig_UniVarPoly} 
\end{figure}

\begin{lemma}
\label{approx_smooth}
    Let $f\in C^{\infty}([0,1])$ satisfy $\left|f^{(n)}\right|\leq n!$ for all $n\in\mathbb{N}^+$, then there is a 3D ReLU NN $\mathcal{N}_{8,N,H}$, with $H=\lceil\frac{1}{2}\left((\log_2 6)N+3\log_2(N+2)+2\log_2 3\right)\rceil$, whose function $\Phi$ satisfies
\begin{equation}
    \sup_{x\in [0,1]}\left|\Phi(x)-f(x)\right|\leq 2^{-N}.
\end{equation}
\end{lemma}

\begin{proof}
    By the Lagrange interpolation theorem, for any $m\in\mathbb{N}^+$ there exists a Chebyshev series $P_{m}$ of degree $m$ such that
\begin{equation}
\label{eq_ChebyInt}
    \sup_{x\in [0,1]}\left|f(x)-P_m(x)\right|\leq \frac{\left\|f^{(m+1)}\right\|_{L^{\infty}([0,1])}}{(m+1)!2^m},
\end{equation}
where $P_{m}$ can be expanded to $P_m=\sum_{k=0}^m c_k x^k$ with $|c_{k}|\leq 2(m+1)3^m$. Using \cref{approx_poly}, let $\Phi(x):[0,1]\to \mathbb{R}$ be implemented by a 3D ReLU network $\mathcal{N}_{8,m-1,H}$, which satisfies 
\begin{equation}
\label{eq_NNapproxCheby}
    \sup_{x\in [0,1]}\left|\Phi(x)-P_m(x)\right|\leq 2(m+1)3^{m+1} m^2 2^{-2(H+1)}\leq (m+1)^3 3^{m+1}2^{-2H-1}
\end{equation}
Setting $H=\left\lceil\frac{1}{2}\left((\log_2 6)m+3\log_2(m+1)+\log_2 3-1\right)\right\rceil$ and $m=N+1$, from Eqs. \eqref{eq_ChebyInt} and \eqref{eq_NNapproxCheby}, we obtain
\begin{equation}
    \sup_{x\in [0,1]}\left|\Phi(x)-f(x)\right|\leq \sup_{x\in [0,1]}\left|\Phi(x)-P_m(x)\right|+ \sup_{x\in [0,1]}\left|f(x)-P_m(x)\right|\leq 2^{-N}.
\end{equation}
    
\end{proof}

\subsection{Approximation of Analytic Functions on the Cube}
\label{sec_analytic_cube}

\begin{theorem}[Main Result 1]
\label{real_analytic}
    For any $0<\delta<1$, let $f$ be a real analytic function with absolutely convergent power series expansion $f=\sum_{\boldsymbol{j}\geq 0} a_{\boldsymbol{j}} \boldsymbol{x}^{\boldsymbol{j}}$ over $[0,1-\delta]^d$, where, without loss of generality, we assume that $\sum_{\boldsymbol{j}\geq 0}|a_{\boldsymbol{j}}|\leq 1$. Then there is a 3D network $\mathcal{N}_{W,N-1,H}$ with $W=\left\lceil d+1+6\left(\frac{e(N+d)}{d}\right)^{d}\right\rceil$ and $$H=\left\lceil\frac{1}{2}\left(\log_{2}N+N\log_2(1-\delta)+d\log_2(N+d)-d\log_2\frac{d}{e}-\log{_2 6}-2\right)\right\rceil,$$ whose function $\Phi:[0,1-\delta]^d\to\mathbb{R}$ satisfies
\begin{equation}
    \sup_{\boldsymbol{x}\in [0,1-\delta]^d}\left|f(\boldsymbol{x})-\Phi(\boldsymbol{x})\right|\leq 2(1-\delta)^N.
\end{equation}
\end{theorem}

\begin{theorem}[Main Result 2]
\label{holomorphic_ellipse}
    Assume that for some $\rho>2^{\sqrt{d}}$, $f:[0,1]^d\to\mathbb{R}$ is bounded and can be analytically continued to the complex ellipse $E_{d,\rho}$ defined in Eq. \eqref{eq_ellipse}. Then there exists a 3D network $\mathcal{N}_{W,N-1,H}$ with $W=\left\lceil d+1+6\left(\frac{e(N+d)}{d}\right)^{d}\right\rceil$ and $H=\left\lceil\frac{1}{2}\left(d\log_2\left((N+1)\left(\frac{e(N+d)}{d}\right)\right)+\log_2 6N +\frac{N}{\sqrt{d}}\log_2\rho-2\right)\right\rceil$, whose function $\Phi:[0,1]^d\to\mathbb{R}$ satisfies
\begin{equation}
    \sup_{\boldsymbol{x}\in[0,1]^d}\left|f(\boldsymbol{x})-\Phi(\boldsymbol{x})\right|\leq C \rho^{-N/\sqrt{d}},
\end{equation}
for some constant $C>0$.
\end{theorem}

Note that the exponential approximation rate of analytic functions on the subinterval $[1-\delta]^d$ is also studied in \cite{wang2018exponential,Beknazaryan2021Analytic}. In \cite{wang2018exponential}, the authors constructed an excessively deep ReLU NN with with fixed width and depth $\mathcal{O}(N^{2d})$ to achieve an error $(1-\delta)^N$. Later, \cite{Beknazaryan2021Analytic} improved the result using a network of width $\mathcal{O}\left(N^{d+2}\right)$, depth $\mathcal{O}\left(N^2\right)$, and height $1$. However, in our \cref{real_analytic}, the width, depth, and height become $\mathcal{O}\left(N^{d-1}\right)$, $\mathcal{O}(N)$, and $\mathcal{O}(N)$, respectively. Although height increases, the network's overall parametric complexity is greatly reduced. A similar improvement is also achieved in the case that the target function $f$ can be analytically continued to the ellipse $E_{d,\rho}$ in $\mathbb{C}^d$ (\cref{holomorphic_ellipse}). Under such a holomorphic condition, the partial sum of a Chebyshev series has a geometric convergence rate to $f$ on $[0,1]^d$ due to the multivariate Bernstein theorem \cite{Trefethen2016Multivariate}, which implies an efficient approximation of $f$ by 3D ReLU NNs. Although proving our two main theorems is conceptually straightforward based on approximating partial sums of power or Chebyshev series by networks, our result surpasses the previous work to achieve the best known approximation rate for the target function classes.

To give the proof of \cref{real_analytic} and \cref{holomorphic_ellipse}, we first extend \cref{approx_poly} to the multivariate case.
\begin{lemma}
    For any $n,H\in\mathbb{N}^+$ and polynomial $P(x)=\sum_{j_1+\ldots+j_d\leq n}a_{\boldsymbol{j}}\boldsymbol{x}^{\boldsymbol{j}}$, there is a 3D ReLU NN $\mathcal{N}_{W,n-1,H}$ with $W=\left\lceil d+1+6\left(\frac{e(n+d)}{d}\right)^{d}\right\rceil$, whose function $\Phi:[0,1]^d\to\mathbb{R}$ fulfills 
\begin{equation}
    \sup_{\boldsymbol{x}\in [0,1]^d}\left|P(\boldsymbol{x})-\Phi(\boldsymbol{x})\right|\leq\max_{j_1+\ldots+j_d\leq n}\left|a_{\boldsymbol{j}}\right|\cdot 6n 2^{-2(H+1)}\left(\frac{e(n+d)}{d}\right)^d.
\end{equation}
\label{poly_multivar}
\end{lemma}

\begin{proof}
    Similar to \cref{approx_poly}, we utilize the product function $\widehat{\times}:[0,1]^2\to [0,1]$ given by a 3D ReLU NN $\mathcal{N}_{6,1,H}$ to approximate each monomial in $P(\boldsymbol{x})$. We set 
\begin{equation}
    h_{0,\ldots,0}=1,\quad h_{1,0,\ldots,0}=x_1,\quad\ldots\quad,h_{0,\ldots,0,1}=x_d,
\end{equation}
where $h_{j_1,\ldots,j_d}:[0,1]^d\to[0,1]$ has $d$ indices, and $j_1+\ldots+j_d\leq n$. Then, similar to \cref{approx_poly}, given $h_{j_1,\ldots,j_k,\ldots,j_d}$, we recursively define $h_{j_1,\ldots,j_k+1,\ldots,j_d}$ by $h_{j_1,\ldots,j_k+1,\ldots,j_d}=\widehat{\times}\left(x_{k},h_{j_1,\ldots,j_k,\ldots,j_d}\right)$. Hence, for all $0\leq \boldsymbol{j}$, $h_{\boldsymbol{j}}$ approximate the monomial $\boldsymbol{x}^{\boldsymbol{j}}$ with an error no more than $6n\cdot 2^{-2(H+1)}$. Note that the number of monomials of degree $k$ is at most $\binom{n+d-1}{d-1}$. Then let $\Phi(x)=\sum_{j_1+\ldots+j_d\leq n}a_{\boldsymbol{j}}h_{\boldsymbol{j}}(\boldsymbol{x})$, and we obtain
\begin{equation}
\begin{split}
    \sup_{\boldsymbol{x}\in [0,1]^d}\left|P(\boldsymbol{x})-\Phi(\boldsymbol{x})\right|&\leq 6\max_{j_1+\ldots+j_d\leq n}\left|a_{\boldsymbol{j}}\right|n 2^{-2(H+1)}\sum_{k=0}^{n}\binom{k+d-1}{d-1}\\
    &= 6\max_{j_1+\ldots+j_d\leq n}\left|a_{\boldsymbol{j}}\right|n 2^{-2(H+1)}\binom{n+d}{d}\\
    &\leq 6\max_{j_1+\ldots+j_d\leq n}\left|a_{\boldsymbol{j}}\right|n 2^{-2(H+1)}\left(\frac{e(n+d)}{d}\right)^d,
\end{split}
\end{equation}
where the third line is due to the Stirling formula.

Finally, we implement $\Phi$ by a 3D ReLU network. In the $k$-th layer, we output all $h_{\boldsymbol{j}}(\boldsymbol{x})$ for $j_1+\ldots+j_d=k+1$, each given from some $h_{\boldsymbol{i}}(\boldsymbol{x})$ with $i_1+\ldots+i_d=k$ and some $x_i$ via the operation of $\hat{\times}(\cdot,\cdot)$. Hence, the $k$-th layer outputs at most $\binom{n+d}{d}$ $h_{\boldsymbol{j}}$. The whole network is of width
\begin{equation}
    d+1+6\binom{n+d}{d}\leq d+1+6 \left(\frac{e(n+d)}{d}\right)^{d},
\end{equation}
depth $n$, and height $H$.
\end{proof}

\begin{proof}[Proof of \cref{real_analytic}]
    Following the same steps as in \cite{wang2018exponential}, we truncate the power series of $f$ to the polynomial $P(\boldsymbol{x})=\sum_{j_1+\ldots+j_d\leq n}a_{\boldsymbol{j}}\boldsymbol{x}^{\boldsymbol{j}}$ of degree no more than $n$, such that
\begin{equation}
    \sup_{\boldsymbol{x}\in[0,1-\delta]^d}\left| f(\boldsymbol{x})-P(\boldsymbol{x}) \right|\leq (1-\delta)^n.
\end{equation}
    Then applying \cref{poly_multivar} to $P(x)$ with $n=N$ and casting $$H=\left\lceil\frac{1}{2}\left(\log_{2}N+N\log_2(1-\delta)+d\log_2(N+d)-d\log_2\frac{d}{e}-\log{_2 6}-2\right)\right\rceil,$$ 
    we immediately finish the proof.
\end{proof}

\begin{proof}[Proof of \cref{holomorphic_ellipse}]
    By the multivariate Bernstein’s theorem (Lemma 3.2, \cite{Beknazaryan2021Analytic}), for any $n>0$, there is a polynomial $P(x)=\sum_{j_1+\ldots+j_d\leq n}a_{\boldsymbol{j}}\boldsymbol{x}^{\boldsymbol{j}},$ such that
\begin{equation}
\label{eq_BernsteinThm}
    \sup_{\boldsymbol{x}\in[0,1]^d}\left| f(\boldsymbol{x})-P(\boldsymbol{x}) \right|\leq C^{\prime}\rho^{-n/\sqrt{d}}
\end{equation}
and
\begin{equation}
\label{eq_BernsteinCoeff}
    \sup_{j_1+\ldots+j_d\leq n}|a_{\boldsymbol{j}}|\leq C(n+1)^d,
\end{equation}
where $C^{\prime}$ is a constant dependent only on $\rho$, $d$ and the bound of $f$.
Then the result follows from \cref{poly_multivar}, Eq. \eqref{eq_BernsteinThm} and Eq. \eqref{eq_BernsteinCoeff} with $n=N$ and 
$$H=\left\lceil\frac{1}{2}\left(d\log_2\left((N+1)\left(\frac{e(N+d)}{d}\right)\right)+\log_2 6N+\frac{N}{\sqrt{d}}\log_2\rho-2\right)\right\rceil.$$

\end{proof}

\subsection{Approximation of Analytic Functions on $L^2(\mathbb{R}^d,\gamma_d)$}
\label{sec_analytic}

\begin{theorem}[Main Result 3]
\label{holomorphic_R^d}
    Let $f:\mathbb{R}^d\to \mathbb{R}$ be holomorphic on the complex open strip $S_{\boldsymbol{\tau}}$ defined in Eq. \eqref{eq_strip}
for some $\boldsymbol{\tau}>0$. Assume that for every $0\leq \boldsymbol{\beta}<\boldsymbol{\tau}$, there exists $B>0$ depending only on $\boldsymbol{\beta}$, such that for all $\boldsymbol{z}=\boldsymbol{x}+\operatorname{i}\boldsymbol{y}\in S_{\boldsymbol{\tau}}$, it holds
\begin{equation}
\label{eq_growth}
    |f(\boldsymbol{z})|\leq B\exp\left(\sum_{j=1}^d\frac{x_j^2}{4}-\frac{1}{\sqrt{2}}|x_j|\sqrt{\beta_j^2-\frac{y_j^2}{2}}\right).
\end{equation}
Then, there exists a 3D ReLU network $\mathcal{N}_{W,N+d-1,H}$ with $W=\max\left\{8Nd,N^d(4+d)\right\}$ and 
\begin{equation*}
    H=\frac{d}{2}N\log_2 \left(1+ 6\sqrt{2N\ln(6N)+4B\sqrt{N}} \right)+\log_2 \left(\frac{5}{4}N\right)+\frac{B\log_2 e}{2}\sqrt{N},
\end{equation*}
whose function $\Phi:\mathbb{R}^d\to\mathbb{R}$ satisfies
\begin{equation}
    \left\|f-\Phi\right\|_{L^2(\mathbb{R}^d,\gamma_d)}\leq C e^{-B\sqrt{N}},
\end{equation}
for some constant $C>0$.
\end{theorem}

Extending the approximation of analytic functions from the cube to the global domain leads to an inferior convergence rate, from $\mathcal{O}\left(\exp(-N)\right)$ to $\mathcal{O}\left(\exp(-\sqrt{N})\right)$. However, this result still outperforms the prior work \cite{schwab2021deep}, where a ReLU NN of depth $\mathcal{O}\left(N\log^2N\right)$ and height $1$ can only reach an error of  $\mathcal{O}\left(\exp\left(-N^{1/3}\right)\right)$.

The condition Eq. \eqref{eq_growth} imposes a growth constraint on $f$, limiting its growth on $x_j$ to be slower than $\exp{\left(\frac{x_j^2}{4}\right)}$ \cite[Assumption 4.1]{schwab2021deep}. Under such an assumption, the partial sum of the expansion of $f$ on the Hermite polynomial basis converges at an exponential rate. Hence, we adopt the Hermite polynomial basis. Furthermore, since in the multivariate case, this expansion $\sum_{0\leq \boldsymbol{\nu}\leq n}\langle f,\Xi_{\boldsymbol{\nu}} \rangle \Xi_{\boldsymbol{\nu}}$ is a linear combination of $\mathcal{O}(n^d)$ products of several univariate Hermite polynomials, the error estimate is more complicated than \cref{real_analytic} and \cref{holomorphic_ellipse}. This is why proving \cref{holomorphic_R^d} is more complicated. 

\cref{Reformula_Hermite} and \cref{hermite_tail} demonstrate the bounds of the Hermite polynomial coefficients and the norm of Hermite polynomials outside $[-M,M]$ for $M>0$. \cref{Hermite_expansion} gives the convergence rate of the Hermite expansion. To bridge the neural network approximation and Hermite polynomial expansion, \cref{multi_prod} approximates the multivariate product using 3D ReLU NNs. Then, \cref{approx_hermite}, \cref{approx_hermite_R^n} and \cref{cor_approx_hermite} approximate the Hermite polynomials by 3D ReLU NNs, based on \cref{multi_prod}, \cref{Reformula_Hermite}, and \cref{hermite_tail}. Finally, our \cref{holomorphic_R^d} follows after combining the results of Hermite expansion and neural network approximation.

% To give the proof of \cref{holomorphic_R^d}, we recall the $n$-th probabilists' Hermite polynomial $\Xi_n(x)$, which has another formulation as follows.

\begin{lemma}[Lemma 2.3, \cite{schwab2021deep}]
\label{Reformula_Hermite}
    The $n$-th Hermite polynomial can be written as
\begin{equation}
\label{eq_HermiteReformula}
    \Xi_n(x)=\sum_{j=0}^n c_{n,j}x^j,
\end{equation}
where the coefficients satisfy $\sum_{j=0}^n |c_{n,j}| \le 6^{n/2}$. Consequently, for all $x\in\mathbb{R}$,
\begin{equation}
\label{eq_HermiteCoeff}
|\Xi_n(x)| \le \big( \sqrt{6} \max\{1,|x|\} \big)^n.
\end{equation}
\end{lemma}

The Hermite polynomial in $L^2(\mathbb{R},\gamma_1)$ has the following property.
\begin{lemma}
\label{hermite_tail}
For any $M\geq 1$ and Hermite polynomial $\Xi_n$, we have
\begin{equation}
    \int_{|x|>M}|\Xi_n(x)|^2d\gamma_1(x)\leq \frac{1}{\sqrt{2\pi}}(2n)!!\left(\sqrt{6}M\right)^{2n}e^{-M^2/2}.
\end{equation}
\end{lemma}

\begin{proof}
Using the coefficient bound Eq. \eqref{eq_HermiteCoeff} and the fact that $$\int_{|x|>M}|x|^{n}e^{-x^2/2}dx\leq n!! M^n e^{-M^2/2}$$ (see \cite[Lemma2.4]{schwab2021deep}),we have the following:
\begin{equation*}
\begin{split}
    \int_{|x|>M}|\Xi_n(x)|^2d\gamma_1(x)&=\frac{1}{\sqrt{2\pi}}\int_{|x|>M}|\Xi_n(x)|^2e^{-x^2/2}dx\\
    &\leq \frac{1}{\sqrt{2\pi}}\int_{|x|>M}|x|^{2n}e^{-x^2/2}dx\\
    &\leq \frac{1}{\sqrt{2\pi}}(2n)!!\left(\sqrt{6}M\right)^{2n}e^{-M^2/2}.
\end{split}
\end{equation*}
\end{proof}

The key idea of proving \cref{holomorphic_R^d} is to consider the expansion of a holomorphic on the Hermite basis and then approximate the Hermite series via 3D networks. The basic property of the expansion on Hermite basis is given by the following lemma:
\begin{lemma}[Theorem 4.5, \cite{schwab2021deep}]
\label{Hermite_expansion}
    Let $f:\mathbb{R}^d\to \mathbb{R}$ be a holomorphic function that satisfies the assumption in \cref{holomorphic_R^d}. Then there are some $C_1,C_2>0$ depending only on $\beta_1,\ldots,\beta_d$ such that it holds that
    
\noindent(1) for all $\boldsymbol{\nu}\in \mathbb{N}^{d}$
\begin{equation}
\label{eq_HermiteExpan1}
    \left|\langle f,\Xi_{\boldsymbol{\nu}}\rangle_{L^2(\mathbb{R}^d,\gamma_d)}\right|\leq \pi^{-d/4}C_1 \exp\left( -\sum_{j=1}^{d} \beta_j \sqrt{2\nu_j+1}\right).
\end{equation}

\noindent(2) 
\begin{equation}
\label{eq_HermiteExpan2}
    \left\|f-\sum_{0\leq\boldsymbol{\nu}\leq n}\langle f,\Xi_{\boldsymbol{\nu}}\rangle \Xi_{\boldsymbol{\nu}}\right\|_{L^2(\mathbb{R}^d,\gamma_d)}\leq C_2 \exp\left(-\left(\prod_{j=1}^d\beta_j\right)^{1/d}\sqrt{n}/\sqrt{2}\right).
\end{equation}
\end{lemma}

On the other hand, to approximate the Hermite series, we extend the bivariate product function $\widehat{\times}(\cdot,\cdot)$ in \cref{bi_prod} to the multivariate case.
\begin{lemma}
\label{multi_prod}
    For any $d\geq 2,H>0,M>0$, there is a 3D network $\widetilde{\prod}_d:[-M,M]^d\to\mathbb{R}$ in $\mathcal{N}_{4+d,d-1,H+1}$ such that 
\begin{equation}
\label{eq_MultiProdErr}
    \sup_{x\in[-M,M]^d}\left|\widetilde{\prod}_d(\boldsymbol{x})-x_1\cdots x_d\right|\leq 6(d-1)M^d 2^{-2(H+1)}.
\end{equation}
\end{lemma}

\begin{proof}
    It suffices to consider the case $M=1$. Otherwise, we normalize $M^d\cdot\widetilde{\prod}_d\left(\boldsymbol{x}/M\right)$. Let us first solve the case $d=2$, which just extends the domain in \cref{bi_prod} from $[0,1]\times [0,1]$ to $[-1,1]\times [-1,1]$. Let $g_s$ be the sawtooth function with $2^s$ “sawteeth" on $[0,1]$ defined by Eq. \eqref{eq_sawtooth}, which is given by a 3D network $\mathcal{N}_{2,1,s}$. Then $\tilde{g}_s=g_{s+1}\left(\frac{x+1}{2}\right)$ can be implemented by a 3D network $\mathcal{N}_{2,1,s+1}$ and satisfies $\tilde{g}_s(x)=g_s(|x|)$ for $x\in[-1,1]$.

Following the similar construction of $f_s$ and $\widehat{\times}(\cdot,\cdot)$ in \cref{bi_prod}, $\widetilde{\prod}_2$ can be implemented by a 3D network $\mathcal{N}_{6,1,H+1}$, such that
\begin{equation}
\label{eq_Prod2Err}    \sup_{x_1,x_2\in[-1,1]}\left|\widetilde{\prod}_2(x_1,x_2)-x_1 x_2\right|\leq 6\cdot2^{-2(H+1)}.
\end{equation}
Next, for $d>2$, we recursively set
\begin{equation}
    \widetilde{\prod}_d(x_1,\ldots,x_d)=\widetilde{\prod}_2\left(\widetilde{\prod}_{d-1}(x_1,\ldots,x_{d-1}),x_d\right),
\end{equation}
and obtain Eq. \eqref{eq_MultiProdErr} by repeatedly using Eq. \eqref{eq_Prod2Err}.
The network implementing $\widetilde{\prod}_d$ is illustrated in \cref{Fig_MultiProd}, which is of width $d+4$, depth $d-1$, and height $H+1$.
\end{proof}

\begin{figure}[H] 
\centering 
\includegraphics[width=1\textwidth]{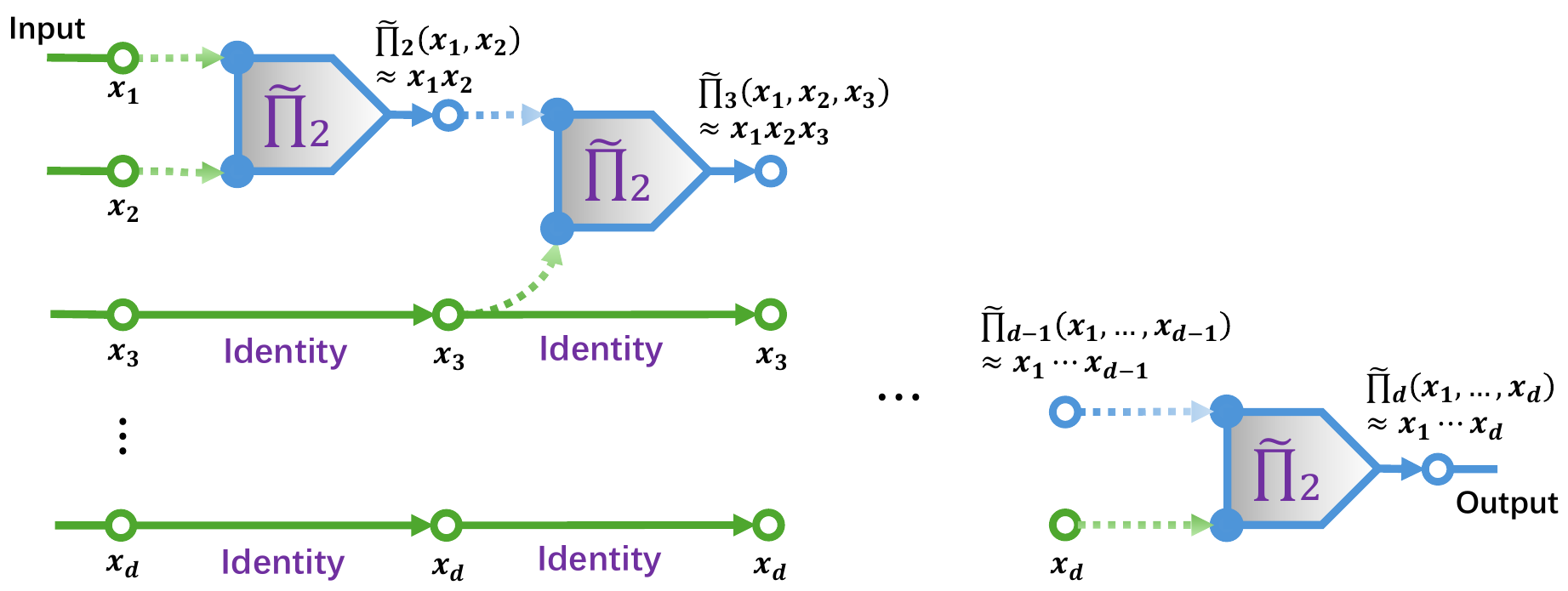} 
\caption{Illustration of the network structure in \cref{multi_prod}} 
\label{Fig_MultiProd} 
\end{figure}

Replacing the bivariate product function $\widehat{\times}(\cdot,\cdot)$ in the construction in \cref{approx_poly} with $\widetilde{\prod}_2(\cdot,\cdot)$ given in \cref{multi_prod}, we immediately obtain the approximation to the Hermite polynomials in different domains as follows:
\begin{lemma}
\label{approx_hermite}
    For any $n\in\mathbb{N}^+$, $M>0$, and $0<\delta<M$, there is a 3D network $\mathcal{N}_{8,n,H+1}$ with $H\geq \frac{n}{2}\log_2 (\sqrt{6}M)+\log_2(3n)-1$, whose function $\widetilde{\Xi}_n:\mathbb{R}\to\mathbb{R}$ fulfills
    
\noindent(i)    

\begin{equation}
\label{eq_HermLemma1}
    \sup_{x\in[-M+\delta,M-\delta]}\left|\widetilde{\Xi}_n(x)-\Xi_n(x)\right|\leq \left(\sqrt{6}M\right)^n 3n^22^{-2(H+1)}\leq 1,
\end{equation}

\noindent(ii) 

\begin{equation}
\label{eq_HermLemma2}
    \widetilde{\Xi}_n(x)=0,\quad\forall x>M \text{ and }x<-M,
\end{equation}

\noindent(iii)
\begin{equation}
\label{eq_HermLemma3}    \sup_{x\in\mathbb{R}}\left|\widetilde{\Xi}_n(x)\right|\leq 1+(\sqrt{6}M)^n,
\end{equation}

\noindent(iv) 
\begin{equation}
\label{eq_HermLemma4}
    \sup_{x\in [-M,-M+\delta]\cup [M-\delta,M]}\left|\widetilde{\Xi}_n(x)-\Xi_n(x)\right|\leq 1+2(\sqrt{6}M)^n.
\end{equation}

\end{lemma}

\begin{proof}
For convenience, we use $\mathfrak{\chi}^1_{M,\delta}$, $\mathfrak{\chi}^2_{M,\delta}$, and $\mathfrak{\chi}^3_{M,\delta}$ to denote the three intervals $[-M+\delta,M-\delta]$, $(-\infty,-M]\cup [M,+\infty)$, and $(-M+\delta,-M)\cup (M-\delta,M)$, respectively. For properties (i) and (ii), we replace the bivariate product function $\widehat{\times}(\cdot,\cdot)$ in the proof of \cref{approx_poly} with $\widetilde{\prod}_2(\cdot,\cdot)$ in \cref{multi_prod}. Thus, using \cref{Reformula_Hermite}, we obtain $\bar{\Xi}_n(x)$ by a 3D ReLU NN $\mathcal{N}_{8,n-1,H+1}$  that approximates the Hermite polynomial $\Xi_n(x)$ with an error
\begin{equation}
    \sup_{x\in\mathfrak{\chi}^1_{M,\delta}}\left|\bar{\Xi}_n(x)-\Xi_n(x)\right|\leq \left(\sqrt{6}M\right)^n 3n^22^{-2(H+1)}.
\end{equation}
To truncate $\bar{\Xi}_n(x)$ to zero outside $\mathfrak{\chi}^1_{M,\delta}$, for $\delta>0$, we let 
\begin{equation}
    \xi(x)=\begin{cases}
        x, &x\in \mathfrak{\chi}^1_{M,\delta}\\
        0, &x\in \mathfrak{\chi}^2_{M,\delta}\\
        \text{linear connection} &x\in \mathfrak{\chi}^3_{M,\delta}
    \end{cases}    
    \quad \text{and}\quad
    \zeta_n(x)=\begin{cases}
        \Xi_n(0), &x\in \mathfrak{\chi}^1_{M,\delta}\\
        0, &x\in \mathfrak{\chi}^2_{M,\delta}\\
        \text{linear connection} &x\in \mathfrak{\chi}^3_{M,\delta}
    \end{cases} 
,
\end{equation}
both of which can be given by a single-layer ReLU network of width $4$. Then we set
\begin{equation}
\label{eq_NNapproxHermite}
    \widetilde{\Xi}_n(x)=\bar{\Xi}_n\left(\xi(x)\right)+\zeta_n(x),
\end{equation}
 which satisfies Eqs. \eqref{eq_HermLemma1} and \eqref{eq_HermLemma2}.
 
For property (iii), note that on $\mathfrak{\chi}^2_{M,\delta}$, $|\widetilde{\Xi}_n(x)|=0$ by Eq. \eqref{eq_HermLemma2}. Hence, by Eq. \eqref{eq_HermLemma1},
\begin{equation}
    \sup_{x\in\mathbb{R}}|\widetilde{\Xi}_n(x)|=\sup_{x\in\mathfrak{\chi}^1_{M,\delta}}|\widetilde{\Xi}_n(x)|\leq \sup_{x\in\mathfrak{\chi}^1_{M,\delta}}|\Xi_n(x)|+|\widetilde{\Xi}_n(x)-\Xi_n(x)|\leq 1+(\sqrt{6}M)^n.
\end{equation}

For property (iv), it follows from Eqs. \eqref{eq_HermLemma1} and \eqref{eq_HermLemma3} that
\begin{equation}
    \sup_{x\in \mathfrak{\chi}^3_{M,\delta}}\left|\widetilde{\Xi}_n(x)-\Xi_n(x)\right|\leq  \sup_{x\in \mathfrak{\chi}^3_{M,\delta}}\left|\widetilde{\Xi}_n(x)\right|+\left|\Xi_n(x)\right| \leq 1+2(\sqrt{6}M)^n.
\end{equation}
\end{proof}

\begin{lemma}
For any Hermite polynomial $\Xi_n$, $M>0$, and $\delta>0$, let $\widetilde{\Xi}_n(x)$ be the function given by Eq. \eqref{eq_NNapproxHermite}, then we have
\label{approx_hermite_R^n}
\begin{equation}
\label{eq_LemmaApproxHermite}
\begin{split}
    \left\| \widetilde{\Xi}_n(x)-\Xi_n(x) \right\|_{L^2(\mathbb{R},\gamma_1)}\leq\left(\sqrt{6}M\right)^n3n^22^{-2(H+1)}+\varepsilon_{n,M}(\delta)+\frac{1}{\sqrt{2\pi}}\left(\sqrt{6n}M\right)^ne^{-M^2/4},
\end{split}
\end{equation}
where $\varepsilon_{n,M}(\delta)\to0$ as $\delta\to 0$.
\end{lemma}

\begin{proof}
By property (i) of \cref{approx_hermite}, we have
\begin{equation}
\label{eq_LemmaApproxHermite1}
\begin{split}
    \left\| \widetilde{\Xi}_n(x)-\Xi_n(x) \right\|_{L^2(\mathfrak{\chi}^1_{M,\delta},\gamma_1)}\leq \left(\sqrt{6}M\right)^n3n^22^{-2(H+1)}.
\end{split}
\end{equation}
Using property (iv) of \cref{approx_hermite}, we obtain
\begin{equation}
\label{eq_LemmaApproxHermite2}
\begin{split}
   \left\| \widetilde{\Xi}_n(x)-\Xi_n(x) \right\|_{L^2(\mathfrak{\chi}^3_{M,\delta},\gamma_1)}\leq 2\sqrt{\delta}\left(1+2\left(\sqrt{6}M\right)^n\right).
\end{split}
\end{equation}
By property (ii) of \cref{approx_hermite} and \cref{hermite_tail},
\begin{equation}
\label{eq_LemmaApproxHermite3}
\begin{split}
\left\| \widetilde{\Xi}_n(x)-\Xi_n(x) \right\|_{L^2(\mathfrak{\chi}^2_{M,\delta},\gamma_1)}&=\left\| \Xi_n(x) \right\|_{L^2(\mathfrak{\chi}^2_{M,\delta},\gamma_1)}\\
&\leq \frac{1}{\sqrt{2\pi}}\sqrt{(2n)!!}\left(\sqrt{6}M\right)^ne^{-M^2/4}\\
&\leq \frac{1}{\sqrt{2\pi}}\left(\sqrt{6n}M\right)^ne^{-M^2/4}.\\
\end{split}
\end{equation}
Hence, combining Eqs. \eqref{eq_LemmaApproxHermite1}, \eqref{eq_LemmaApproxHermite2}, and  \eqref{eq_LemmaApproxHermite3} yields
    \begin{equation}
\begin{split}
    &\left\| \widetilde{\Xi}_n(x)-\Xi_n(x) \right\|_{L^2(\mathbb{R},\gamma_1)}\\
    \leq& \left\| \widetilde{\Xi}_n(x)-\Xi_n(x) \right\|_{L^2(\mathfrak{\chi}^1_{M,\delta},\gamma_1)}+\left\| \Xi_n(x) \right\|_{L^2(\mathfrak{\chi}^2_{M,\delta},\gamma_1)}+\left\| \widetilde{\Xi}_n(x)-\Xi_n(x) \right\|_{L^2(\mathfrak{\chi}^3_{M,\delta},\gamma_1)}\\
    \leq &\left(\sqrt{6}M\right)^n3n^22^{-2(H+1)}+2\sqrt{\delta}\left(1+2\left(\sqrt{6}M\right)^n\right)+\frac{1}{\sqrt{2\pi}}\left(\sqrt{6n}M\right)^ne^{-M^2/4}.
\end{split}
\end{equation}
\end{proof}

With some special choice of $M$ and $H$ in Eq. \eqref{eq_LemmaApproxHermite}, we have the following corollary:
\begin{corollary}
\label{cor_approx_hermite}
    For any $B>0$, let $M=\sqrt{12n\ln(6n)+24B\sqrt{n}}$, $H=\frac{1}{2}n\log{(\sqrt{6}M)}+\log_2 (\frac{5}{4}n)+\frac{B\log_2 e}{2}\sqrt{n}$ and $\delta>0$ small enough in \cref{approx_hermite_R^n}, then we have
\begin{equation}
    \left\| \widetilde{\Xi}_n(x)-\Xi_n(x) \right\|_{L^2(\mathbb{R},\gamma_1)}\leq \left(\frac{3}{5}+\frac{1}{\sqrt{2\pi}}\right)e^{-B\sqrt{n}}+\varepsilon_{n,M}(\delta)\leq e^{-B\sqrt{n}}.
\end{equation}
\end{corollary}

\begin{proof}
We first show that with $M=\sqrt{12n\ln(6n)+24B\sqrt{n}}$, we have
\begin{equation}
\label{eq_CorExpIneq}
    \frac{1}{\sqrt{2\pi}}\left(\sqrt{6n}M\right)^ne^{-M^2/4}\leq e^{-B\sqrt{n}}.
\end{equation}
It suffices to show that
\begin{equation}
\label{eq_CorIneq}
\begin{split}
    -\frac{M^2}{4}+\frac{n}{2}\ln(6n)+n\ln M+B\sqrt{n}\leq 0,
\end{split}
\end{equation}
On one hand, we have 
\begin{equation}
\label{eq_CorIneq1}
    \frac{n}{2}\ln(6n)+B\sqrt{n}\leq \frac{M^2}{24}.
\end{equation}
On the other hand, we show
\begin{equation}
\label{eq_CorIneq2}
     n\ln{M}\leq \frac{M^2}{6}.
\end{equation}
Then summating Eqs. \eqref{eq_CorIneq1} and \eqref{eq_CorIneq2} together yields Eq. \eqref{eq_CorIneq}.
In fact, note that the function $x\mapsto \frac{x^2}{\ln{x}}$ is monotonically increasing for all $x\geq \sqrt{e}$. Hence, it suffices to show that
\begin{equation}
    6n\leq \frac{12n\ln(6n)}{\ln\left(12n\ln(6n)\right)},
\end{equation}
which follows directly from
\begin{equation}
    \ln\left( 12n\ln(6n) \right)\leq \ln\left(12n\cdot\frac{6}{e}n\right) \leq 2\ln{6n}.
\end{equation}
Now with Eq. \eqref{eq_CorExpIneq}, we take $H=\frac{1}{2}n\log_2{(\sqrt{6}M)}+\log_2 (\frac{5}{4}n)+\frac{B\log_2 e}{2}\sqrt{n}$. Then we have
\begin{equation}
    \left(\sqrt{6}M\right)^n3n^22^{-2(H+1)}=\frac{3}{5}e^{-B\sqrt{n}}.
\end{equation}
Hence, it follows that
\begin{equation}
    \left\| \widetilde{\Xi}_n(x)-\Xi_n(x) \right\|_{L^2(\mathbb{R},\gamma_1)}\leq \left(\frac{3}{5}+\frac{1}{\sqrt{2\pi}}\right)e^{-B\sqrt{n}}+\varepsilon_{n,M}(\delta)\leq e^{-B\sqrt{n}}
\end{equation}
for some $\delta$ small enough.
\end{proof}

With these results, we are ready to prove \cref{holomorphic_R^d}.

\begin{proof}[Proof of \cref{holomorphic_R^d}]
Suppose that $\widetilde{\Xi}_{\nu_j}\in \mathcal{N}_{8,n,H_1+1}$ is given by \cref{cor_approx_hermite}, and  $\widetilde{\prod}_d\in \mathcal{N}_{4+d,d-1,H_2+1} $ is given by \cref{multi_prod}. Then
\begin{equation}
\label{eq_E1E2}
\begin{split}
    &\left\|\Xi_{\nu_1}\cdots \Xi_{\nu_d}-\widetilde{\prod}_d\left(\widetilde{\Xi}_{\nu_1},\ldots,\widetilde{\Xi}_{\nu_d}\right)\right\|_{L^2(\mathbb{R}^d,\gamma_d)}\\
    \leq &\left\|\Xi_{\nu_1}\cdots \Xi_{j_d}-\widetilde{\Xi}_{\nu_1}\cdots\widetilde{\Xi}_{\nu_d}\right\|_{L^2(\mathbb{R}^d,\gamma_d)}+\left\|\widetilde{\Xi}_{\nu_1}\cdots\widetilde{\Xi}_{\nu_d}-\widetilde{\prod}_d\left(\widetilde{\Xi}_{\nu_1},\ldots,\widetilde{\Xi}_{\nu_d}\right)\right\|_{L^2(\mathbb{R}^d,\gamma_d)}\\
    =&E_1+E_2.
\end{split}
\end{equation}
For the second term in the RHS of Eq. \eqref{eq_E1E2}, we have
\begin{equation}
\label{eq_E2}
    E_2\leq 6(d-1)\left(1+\left(\sqrt{6}M\right)^n\right)^d2^{-2(H_2+1)}\leq \frac{3}{2}(d-1)(1+\sqrt{6}M)^{nd}2^{-2H_2}.
\end{equation}
Taking $H_2=\frac{d}{2}n\log_2(1+\sqrt{6}M)+\frac{B\log_2e}{2}\sqrt{n}$ in Eq. \eqref{eq_E2}, we have
\begin{equation}
    E_2\leq \frac{3}{2}(d-1)e^{-B\sqrt{n}}.
\end{equation}

For the first term in the RHS of Eq. \eqref{eq_E1E2}, we use the triangular inequality and the relationship 

\begin{equation}
    \prod_{j=1}^{d}a_j-\prod_{j=1}^{d}b_j =\sum_{j=1}^{d} \left(a_j-b_j\right)\prod_{i<j}a_i \prod_{j<i\leq d}b_i
\end{equation}
for all $a_i,b_i\in\mathbb{R}$. Then we obtain
\begin{equation}
\begin{split}
    E_1&\leq \sum_{j=1}^{d} \left\|\Xi_{\nu_j}-\widetilde{\Xi}_{\nu_j}\right\|_{L^2(\mathbb{R},\gamma_1)}\prod_{i<j}\left\|\Xi_{\nu_i}\right\|_{L^2(\mathbb{R},\gamma_1)} \prod_{j<i\leq d}\left\|\widetilde{\Xi}_{\nu_i}\right\|_{L^2(\mathbb{R},\gamma_1)}\\
    &\leq d\cdot err (1+err)^{d-1}\\
    &\leq 2^dd\cdot err,
\end{split}
\end{equation}
where $err=\left(\sqrt{6}M\right)^n3n^22^{-2(H_1+1)}+\varepsilon_{n,M}(\delta)+\frac{1}{\sqrt{2\pi}}\left(\sqrt{6n}M\right)^ne^{-M^2/4}\leq e^{-B\sqrt{n}}\leq 1$ and $B=\left(\prod_{j=1}^{d}\beta_j\right)^{1/d}/\sqrt{2}$.
Let
\begin{equation}
    \Phi(\boldsymbol{x})=\sum_{0\leq\boldsymbol{\nu}\leq n}\langle f,\Xi_{\boldsymbol{\nu}}\rangle \widetilde{\Xi}_{\boldsymbol{\nu}}(\boldsymbol{x}),
\end{equation}
where 
\begin{equation}
    \widetilde{\Xi}_{\boldsymbol{\nu}}(\boldsymbol{x})=\widetilde{\prod}_d\left(\widetilde{\Xi}_{\nu_1}(x_1),\ldots,\widetilde{\Xi}_{\nu_d}(x_d)\right).
\end{equation}
Then we have
\begin{equation}
\label{eq_NNApproxHolo}
\begin{split}
    \left\| f-\Phi \right\|_{L^2\left(\mathbb{R}^d,\gamma_d\right)}\leq \left\|f-\sum_{0\leq\boldsymbol{\nu}\leq n}\langle f,\Xi_{\boldsymbol{\nu}}\rangle \Xi_{\boldsymbol{\nu}}\right\|_{L^2(\mathbb{R}^d,\gamma_d)}+\left\|\Phi-\sum_{0\leq\boldsymbol{\nu}\leq n}\langle f,\Xi_{\boldsymbol{\nu}}\rangle \Xi_{\boldsymbol{\nu}}\right\|_{L^2(\mathbb{R}^d,\gamma_d)}.
\end{split}
\end{equation}
For the second term in Eq. \eqref{eq_NNApproxHolo}, we have
\begin{equation}
\begin{split}
\label{eq_HermiteProj}
    \left\|\Phi-\sum_{0\leq\boldsymbol{\nu}\leq n}\langle f,\Xi_{\boldsymbol{\nu}}\rangle \Xi_{\boldsymbol{\nu}}\right\|_{L^2(\mathbb{R}^d,\gamma_d)}&\leq (E_1+E_2)\sum_{0\leq\boldsymbol{\nu}\leq n} \left|\langle f,\Xi_{\boldsymbol{\nu}}\rangle_{L^2(\mathbb{R}^d,\gamma_d)}\right|.
\end{split}
\end{equation}
Now we estimate the summation $\sum_{0\leq\boldsymbol{\nu}\leq n} \left|\langle f,\Xi_{\boldsymbol{\nu}}\rangle_{L^2(\mathbb{R}^d,\gamma_d)}\right|$ in Eq. \eqref{eq_HermiteProj}. By Eq. \eqref{eq_HermiteExpan1}, we have
\begin{equation}
\label{eq_ProjCoeff}
\begin{split}
    \sum_{0\leq\boldsymbol{\nu}\leq n} \left|\langle f,\Xi_{\boldsymbol{\nu}}\rangle_{L^2(\mathbb{R}^d,\gamma_d)}\right|&\leq C_1 \sum_{0\leq\boldsymbol{\nu}\leq n} \exp\left(-\sum_{j=1}^d\beta_j\sqrt{2\nu_j+1} \right)\\
    &\leq C_1 \prod_{j=1}^d\left( \sum_{\nu_j=0}^{n}e^{-\beta_j \sqrt{2\nu_j+1}} \right)\\
    &\leq C_1 \prod_{j=1}^d\left( e^{-\beta_j}+\int_{0}^n e^{-\beta_j\sqrt{2x+1}}dx \right)\\
    &= C_1 \prod_{j=1}^d \left( e^{-\beta_j}\frac{\beta_j^2+\beta_j+1}{\beta_j^2}-e^{-\beta_j\sqrt{2n+1}}\frac{\beta_j\sqrt{2n+1}+1}{\beta_j^2}\right)\\
    &\leq C_1\prod_{j=1}^d e^{-\beta_j}\frac{\beta_j^2+\beta_j+1}{\beta_j^2}=C(\beta).
\end{split}
\end{equation}
Hence, combining Eqs. \eqref{eq_NNApproxHolo},  \eqref{eq_HermiteProj},  \eqref{eq_ProjCoeff}, and \cref{Hermite_expansion}, we have
\begin{equation*}
\begin{split}
    \left\| f-\Phi \right\|_{L^2\left(\mathbb{R}^d,\gamma_d\right)}\leq \left(C_2+\frac{3(d-1)}{2}C(\beta)+C(\beta)2^dd\right)\exp\left(-B\sqrt{n}\right).
\end{split}
\end{equation*}
The 3D network implementing $\Phi$ is illustrated in \cref{Fig_HoloNet}, which is of width $\max\left\{8nd,n^d(4+d)\right\}$, depth $n+d-1$, and height
\begin{equation}
    \left\lceil\frac{d}{2}n\log_2 \left(1+ 6\sqrt{2n\ln(6n)+4B\sqrt{n}} \right)+\log_2 \left(\frac{5}{4}n\right)+\frac{B\log_2 e}{2}\sqrt{n}\right\rceil.
\end{equation}
Finally, the result follows immediately by casting $n=N$.
\end{proof}

\begin{figure}[H] 
\centering 
\includegraphics[width=1\textwidth]{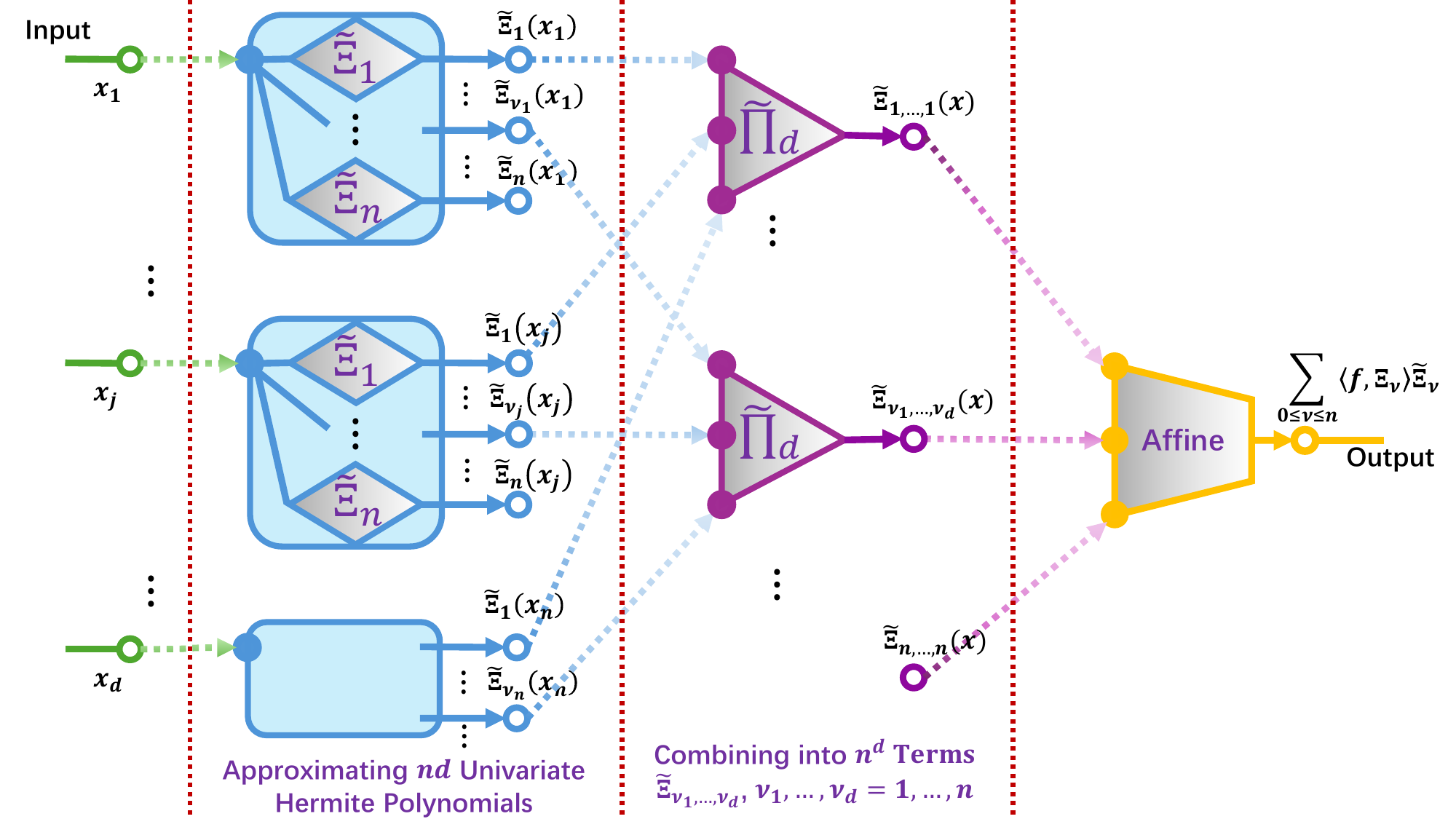} 
\caption{Implementation of the network approximation to analytic functions in \cref{holomorphic_R^d}.} 
\label{Fig_HoloNet} 
\end{figure}

\section{Approximation of $L^p$ Functions}
\label{sec_Lp}
In this section, we establish a quantitative and non-asymptotic approximation to $L^p$ functions. 
\begin{theorem}[Main Result 4]
\label{approx_Lp}
    Let $f:[-1,1]^d\to\mathbb{R}$ be an $L^p$ function and let $r\in \mathbb{N}^+$. Then for any $N_1,N_2>0$, there is a 3D ReLU NN $\mathcal{N}_{W,K,H}$ with $W={(2N_1)}^d(4+d)$, $K=N_2+d$, and $H=\max\left\{ N_2+1
    +\lceil\log_2(N_2+1)\frac{\log_2 3}{2}\rceil,\lceil\log_2 N_1\rceil \right\}$, whose function $\Phi$ fulfills
\begin{equation}
    \left\|f-\Phi\right\|_p\leq r^d B_r \omega_r^d\left(f,N_1^{-1}\right)_p+\frac{3d}{2}\left\|f\right\|_p\left(4C_r N_1\right)^d 2^{-N_2}
\end{equation}
for some constant $B_r$ and $C_r$ dependent only on $r$.
\end{theorem}

\cref{approx_Lp} gives a quantitative and non-asymptotic approximation to general $L^p$ functions on $[-1,1]^d$. Especially, when $\omega_r^d(f,t)_p\leq t^{\alpha}$ for some $\alpha>0$, by taking $N_1=N$ and $N_2=(\alpha+d)\log_2 N$, it follows that a 3D ReLU NN of width $\mathcal{O}(N^d)$, depth $\mathcal{O}(\log_2 N)$, and height $\mathcal{O}(\log_2 N)$ can approximate $f$ with $L^p$ error $\mathcal{O}(N^{-\alpha})$. The proof relies on the trigonometric approximation of $L^p$ functions, which is constructed by a kernel trigonometric polynomial \cite{alma1993Constructive}. The key is to give an explicit bound of the trigonometric polynomials for the quantitative and non-asymptotic analysis.

\cref{approx_tri} approximates the trigonometric functions using a 3D ReLU NN. \cref{prop_jackson_moment} gives the explicit construction of a generalized Jackson-type kernel, a trigonometric polynomial for further construction of the kernel polynomial. \cref{Lp_poly_approx} constructs the trigonometric approximation of the univariate $L^p$ function, based on the kernel in \cref{prop_jackson_moment}.  \cref{poly_approx_multivar} extends the trigonometric approximation to the multivariate case and constructs a 3D ReLU NN to approximate the $L^p$ function that is even or odd in each variable. Finally, we decompose the $L^p$ function $f$ into $f=\sum_{i}f_i$, where $f_i$ is either even or odd in each variable, and apply \cref{poly_approx_multivar} to prove \cref{approx_Lp}.
% \begin{corollary}
%     Assume that $f\in L^p\left([-1,1]^d\right)$ satisfies $\omega_r^d(f,t)_p\leq t^{\alpha}$ for some $\alpha>0$. Then there is a 3D ReLU NN $\Phi:[-1,1]^d\to \mathbb{R}$ of width ${(2N)}^d(4+d)$, depth $(d+\alpha)\log_2 N+d$ and height $\lceil(d+\alpha)\log_2 N+\log_2\left( (\alpha+d)\log_2N+1 \right)+\frac{\log_2 3}{2}\rceil+1$ such that
% \begin{equation}
%     \left\|f-\Phi\right\|_p\leq \left(r^d B_r +\frac{3d}{2}\left(4C_r\right)^d\right)N^{-\alpha}.
% \end{equation}
% \end{corollary}

% \begin{proof}
%     The result follows directly from \cref{approx_Lp} with $L=(\alpha+d)\log_2 N$.
% \end{proof}

%\subsection{Approximation to $L^p$ Functions}
\begin{lemma}
\label{approx_tri}
    For any $N,k\in\mathbb{N}^+$, there is a 3D ReLU NN $\mathcal{N}_{8,N+1,H}$ with $$H=\max\left\{N+1+\lceil \log_2(N+1)+\frac{1}{2}\log_2 3 \rceil,\lceil \log_2 k \rceil\right\},$$ whose function $\Psi_{\cos}:[-1,1]\to\mathbb{R}$ satisfies
\begin{equation}
\label{eq_NNApproxTri}
    \sup_{x\in [-1,1]}\left|\Psi_{\cos}(x)-\cos (k\pi x)\right|\leq 2^{-N}.
\end{equation}
Similarly, $\sin (kx)$ can also be approximated with the same error by a 3D ReLU NN of the same size.
\end{lemma}

\begin{proof}
   By \cref{approx_smooth}, there is a 3D network $\mathcal{N}_{8,N,N+1+\lceil \log_2(N+1)+\frac{1}{2}\log_2 3 \rceil}$, whose function $\widetilde{\Psi}:[-1,1]\to \mathbb{R}$ satisfies
\begin{equation}
\begin{split}
    \sup_{x\in [-1,1]} \left|\widetilde{\Psi}(x)-\cos (\pi x)\right|&\leq 2^{-N} .
\end{split}
\end{equation}
For $\cos (k\pi x)$, we set $s=\lceil \log_2 k \rceil$. Then for $x\in [-1,1]$, we have
\begin{equation}
    \cos (k\pi x) = \cos\left(\pi g_s\left( \frac{k}{2^s} |x| \right)\right).
\end{equation}
Note that $g_s\left(\frac{k}{2^s}|x|\right)\in [-1,1]$, we have
\begin{equation}
    \sup_{x\in [0,1]}\left|\widetilde{\Psi}\circ g_s\left( \frac{k}{2^s}|x| \right)-\cos (k\pi x)\right|\leq 2^{-N}.
\end{equation}
Since $g_s(|x|)$ can be implemented by a single-layer 3D ReLU NN of width $2$ and height $s$. Hence, $\Psi_{\cos}=\widetilde{\Psi}\circ g_s\left(|x|\right)$ can be represented by a 3D network of width $8$, depth $N+1$, and height $\max\left\{N+1+\lceil \log_2(N+1)+\frac{1}{2}\log_2 3 \rceil,\lceil \log_2 k \rceil\right\}$.
Similarly, for any $x\in [-1,1]$, we have
\begin{equation}
    \sin (k\pi x) = \sin\left(\pi  g_s\left( \frac{k}{2^s}  |x|\right)\operatorname{sign}(x)\right),
\end{equation}
which can also be implemented by a 3D ReLU network of the same size and achieve the same approximation error of $\sin (k\pi x)$.
\end{proof}

The following construction is the generalized Jackson kernel \cite{alma1993Constructive}[Lemma 2.1, Chapter 7]. Here we give an explicit formulation of the kernel and the bound estimation.

\begin{lemma}
\label{prop_jackson_moment}
    Let
\begin{equation}
\label{eq_Kernel}
    K_{n,r}(t)=\gamma_{n,r}\left(\frac{\sin mt/2}{\sin t/2}\right)^{2r}=\sum_{k=1}^{n}a_{k,r}\cos kt
\end{equation}
    be a trigonometric polynomial of degree $n$, where $n\geq r>0,m=\lfloor n/r\rfloor+1$, and $\gamma_{n,r}$ is defined by the relation
$\int_{-\pi}^{\pi}K_{n,r}(t)dt=1$. Then for some constant $B_r,M_r>0$, we have
\begin{equation}
    \int_{-\pi}^{\pi} t^k K_{n,r}(t)dt\leq B_r n^{-k},\quad k\leq 2r-2,
\end{equation}
and $a_{k,r}\leq M_r n$.
\end{lemma}

\begin{proof}
    Consider the Fejer kernel
\begin{equation}
    F_m(t)=\left(\frac{\sin mt/2}{\sin t/2}\right)^{2}=\sum_{|k|\leq m-1}\left( 1-\frac{|k|}{m} \right)e^{ikt}.
\end{equation}
We recursively define the vector $\boldsymbol{c}_j=(c_{j,1},\ldots,c_{2j(m-1)+1})\in \mathbb{R}^{2j(m-1)+1}$ by
\begin{equation}
    \boldsymbol{c}_j=\begin{cases}
        \left(1-\frac{|k|}{m}\right)_{|k|\leq m-1} &, j=1\\
        \boldsymbol{c}_{1} \ast \boldsymbol{c}_{j-1} &, j>1
    \end{cases}
    ,
\end{equation}
where $\boldsymbol{a}\ast \boldsymbol{b}$ is the discrete convolution of two vectors $\boldsymbol{a}$ and $\boldsymbol{b}$ with step size 1. Then the elements of $\boldsymbol{c}_j$ are the coefficients of $\left(F_m(t)\right)^j$, \textit{i.e.}, $\left(F_m(t)\right)^j=\sum_{|k|\leq j(m-1)}c_{j,k} e^{ikt}$.
Hence, $\left(F_m(t)\right)^r$ can be expanded as a cosine series with positive coefficients:
\begin{equation}
    \left(F_m(t)\right)^r=\sum_{k=1}^n \tilde{a}_{k,r} \cos{kt},\tilde{a}_{k,r}>0.
\end{equation}
Then, we have
\begin{equation}
    \sum_{k=1}^{n}\tilde{a}_{k,r}=\left(F_m(0)\right)^r=m^{2r},
\end{equation}
and hence $|\tilde{a}_{k,r}|<m^{2r}$.  Note that
\begin{equation}
\begin{split}
    \gamma_{n,r}&=\left(2\int_{0}^{\pi}\left(\frac{\sin mt/2}{\sin t/2}\right)^{2r}dt\right)^{-1}\\
    &\sim \left(m^{2r-1}\int_{0}^{m\pi/2} \left(\frac{\sin \tau}{\tau}\right)^{2r}d\tau\right)^{-1}\\
    &\sim n^{-(2r-1)},
\end{split}
\end{equation}
where $x_n\sim y_n$ means $x_n$ and $y_n$ are infinitesimals of the same order as $n\to\infty$.
Therefore, it follows that
\begin{equation}
    a_{k,r}=\tilde{a}_{k,r}\gamma_{n,r}\leq \mathcal{O}\left(m^r n^{-2r+1}\right) \leq M_r n
\end{equation}
for some constant $M_r>0$. Similarly, for $0<k\leq 2r-2$, we have
\begin{equation}
\begin{split}
    \int_{0}^{\pi}t^k K_{n,r}(t)dt &\sim \gamma_{n,r}\int_{0}^{\pi}t^k \left(\frac{\sin mt/2}{t}\right)^{2r}dt\\
    &\sim n^{2r-k-1}\int_{0}^{m\pi/2}\tau^k \left(\frac{\sin \tau}{\tau}\right)^{2r}d\tau\\
    &\leq B_r n^{-k}.
\end{split}
\end{equation}

\end{proof}

\begin{lemma}
    \label{Lp_poly_approx}
    Let $f:[-\pi,\pi]\to\mathbb{R}$ be an $L^p$ function. For any fixed $r\in \mathbb{N}^+$ and any $n\geq r$, let $\mathcal{T}_n:L^p([-\pi,\pi])\to \Pi_n$ be the linear operator defined by 
\begin{equation}
    \mathcal{T}_n(f)(x)=\int_{-\pi}^{\pi}\left[(-1)^{r+1}\Delta_t^r \left(f,x\right)+f\left(x\right)\right]K_{n,r}(t)dt,
\end{equation}
where $K_{n,r}(t)$ is given by \cref{prop_jackson_moment}. Then we have (i) $\left\| \mathcal{T}_n(f)-f \right\|_{p}\leq B_r \omega_r\left(f,n^{-1}\right)_p$ for some constant $B_r$; (ii) $\mathcal{T}_n(f)(x)$ can be written as a trigonometric series
\begin{equation}
    \mathcal{T}_n(f)(x)=\sum_{j=0}^{n}\alpha_j \int_{-\pi}^{\pi}f\left(u\right)\cos{(ju)}du \cos{(j x)}+\alpha_j \int_{-\pi}^{\pi}f\left(u\right)\sin{(ju)}du \sin{(j x)},
\end{equation}
with $\left| \alpha_j \right|\leq C_r n$ for some constant $C_r$ dependent only on $r$; (iii)  $\left\| \mathcal{T}_n (f)\right\|_p \leq r \left\| f\right\|_p$.

\end{lemma}

\begin{proof}
For property (i), using the fact that
\begin{equation}
    \omega_r\left(f,t\right)_p\leq (nt+1)^r\omega_r\left(f,n^{-1}\right)_p,
\end{equation}
the generalized Minkovski inequality, and \cref{prop_jackson_moment}, we have
\begin{equation}
\begin{split}
    \left\| \mathcal{T}_n(f)-f \right\|_{p} &\leq \left\| \int_{-\pi}^{\pi}\Delta_t^r \left(f,x\right)K_{n,r}(t)dt \right\|_p\\
    &\leq \int_{-\pi}^{\pi}\omega_r \left(f,|t|\right)_pK_{n,r}(t)dt \\
    &\leq \omega_r\left(f,n^{-1}\right)_p \int_{-\pi}^{\pi}(nt+1)^r K_{n,r}(t)dt\\
    &\leq B_r \omega_r\left(f,n^{-1}\right)_p.
\end{split}
\end{equation}
For property (ii), note that $\mathcal{T}_n(f)$ can be rewritten as
\begin{equation}
\label{eq_TnfRe}
\begin{split}
    \mathcal{T}_n(f)(x)&=\int_{-\pi}^{\pi}\left(\sum_{k=1}^{r}(-1)^{k+1}\binom{r}{k}f\left(x+kt\right)\right) K_{n,r}(t)dt\\
    &=\sum_{k=1}^r\sum_{l=1}^n(-1)^{k+1}\binom{r}{k}a_{l,r}\int_{-\pi}^{\pi}f\left(x+kt\right)\cos{lt}dt\\
    &=\sum_{\substack{l=0,\ldots,n,\\k=1,\ldots,r,\\k \text{ divides }l}}\frac{(-1)^{k+1}}{k}\binom{r}{k}a_{l,r}\left(\int_{-\pi}^{\pi}f\left(u\right)\cos{\frac{l}{k}u}du\cos{\frac{l}{k}x}+\int_{-\pi}^{\pi}f\left(u\right)\sin{\frac{l}{k}u}du\sin{\frac{l}{k}x}\right),
\end{split}
\end{equation}
where in the last equation we set $u=x+kt$.
For convenience of notation, we omit the subscript $r$ and denote for $j=0,\ldots,n$,
\begin{equation}
    \alpha_j=\sum_{k=1}^{\min\left\{r,\lfloor\frac{n}{j}\rfloor\right\}}a_{jk,r}\frac{(-1)^{k+1}}{k}\binom{r}{k}.
\end{equation}
Then we rewrite Eq. \eqref{eq_TnfRe} as
\begin{equation}
    \mathcal{T}_n(f)(x)=\sum_{j=0}^{n}\alpha_j \int_{-\pi}^{\pi}f\left(u\right)\cos{(ju)}du \cos{(j x)}+\alpha_j \int_{-\pi}^{\pi}f\left(u\right)\sin{(ju)}du \sin{(j x)}.
\end{equation}
By \cref{prop_jackson_moment}, we have 
\begin{equation}
\begin{split}
    \left| \alpha_j \right|\leq M_r n \sum_{k=1}^{r} \frac{(-1)^{k+1}}{k}\binom{r}{k}.
\end{split}
\end{equation}
% Hence, $\mathcal{T}_n(f)(x)$ can be written as a trigonometric series
% \begin{equation}
%     \mathcal{T}_n(f)(x)=\sum_{j=0}^n c_j \cos 2j\pi x+s_j \sin 2j\pi x
% \end{equation}
% with $c_j,s_j$ bounded by
% \begin{equation}
%     \max_{j=0,\ldots,n} \left\{|c_j|,|s_j|\right\}\leq C_r n \left\|f\right\|_1
% \end{equation}
% for some constant $C_r$ depending only on $r$.
For property (iii), we use $\bar{\Delta}_t^r$ to denote the operator
\begin{equation}
    \bar{\Delta}_t^r(f)(x)=(-1)^{r+1}\Delta_t^r(f,x)+f(x),\quad \forall x\in [-\pi,\pi].
\end{equation}
Since we have
\begin{equation}
    \left\| \bar{\Delta}_t^r(f) \right\|_p\leq r \left\|f\right\|_p,
\end{equation}
then by the generalized Minkovski inequality, 
\begin{equation}
\begin{split}
    \left\|\mathcal{T}_n(f)\right\|_p &=\left(\int_{-\pi}^{\pi} \left|\int_{-\pi}^{\pi}\bar{\Delta}_t^r(f)(x) K_{n,r}(t)dt\right|^p dx\right)^{1/p}\\
    &\leq  \int_{-\pi}^{\pi}\left( \left|\bar{\Delta}_t^r(f)(x) K_{n,r}(t)\right|^p dx \right)^{1/p}dt\\
    &= \int_{-\pi}^{\pi}\left( \left|\bar{\Delta}_t^r(f)(x)\right|^p  dx \right)^{1/p} K_{n,r}(t) dt\\
    &\leq r \left\|f\right\|_p.
\end{split}
\end{equation}
\end{proof}

\begin{lemma}
\label{poly_approx_multivar}
    For any $f\in L^p([-1,1]^d)$, there is a linear operator $\mathcal{T}^d_n:  L^p([-1,1])^d \to \Pi_n^d$, such that (i)
 $ \left\|\mathcal{T}_n^d(f)-f\right\|_p\leq r^d B_{r} \omega_r^d (f,n^{-1})$ for some constant $B_{r}$ depending only on $r$; (ii) especially, if $f$ is either even or odd for each variable, then
\begin{equation}
    \mathcal{T}_n^d(f)=\sum_{0\leq \boldsymbol{j}\leq n} a_{j_1,\ldots,j_d} \prod_{k=1}^d \cos(j_k\pi x_k-\frac{\eta_k\pi}{2}),
\end{equation}
where $\eta_k\in \{0,1\}$ indicates the parity of $f$ with respect to the $k$-th variable such that
\begin{equation}
    f(x_1,\ldots,-x_k,\ldots,x_d)=(-1)^{\eta_k}f(x_1,\ldots,x_k,\ldots,x_d),
\end{equation}
and $\left|a_{\boldsymbol{j}}\right|\leq (C_r n)^d \left\|f\right\|_1$ for some constant
$C_{r}$ dependent only on $r$; (iii) under the assumption of (ii), for any $N_1,N_2\in\mathbb{N}^+$, there exists a 3D ReLU NN $\mathcal{N}_{W,K,H}$ with $W=\max\left\{8N_1d,N_1^d(4+d)\right\}$, $K=N_2+d$, and $$H=\max\left\{ N_2+1+\lceil\log_2(N_2+1)+\frac{\log_2 3}{2}\rceil,\lceil\log_2 N_1\rceil \right\},$$ whose function $\Phi:[-1,1]^d\to \mathbb{R}$ satisfies 
\begin{equation}
    \left\|f-\Phi\right\|_p\leq r^d B_r \omega_r^d\left(f,N_1^{-1}\right)+\frac{3d}{2}\left\|f\right\|_1\left(4C_r N_1\right)^d 2^{-N_2}.
\end{equation}
\end{lemma}

\begin{proof}
   
    For property (i), let $\mathcal{T}_{n,k}$ be the operator defined in \cref{Lp_poly_approx} for the $k$-th variable after a scaling $x_k\mapsto x_k/\pi$ and set
\begin{equation}
    \mathcal{T}_n^k=\mathcal{T}_{n,1} \circ \mathcal{T}_{n,2}\circ \cdots \circ \mathcal{T}_{n,k}.
\end{equation}
Then by properties (i) and (iii) of \cref{Lp_poly_approx}, we have
\begin{equation}
\begin{split}
    \left\|\mathcal{T}_n^d(f)-f\right\|_p&\leq  \sum_{k=1}^d \left\|\mathcal{T}_n^{k} (f)-\mathcal{T}_n^{k-1} (f)\right\|_p\\
    &\leq \sum_{k=1}^dr^{k-1} \left\|\mathcal{T}_{n,k} (f)-f\right\|_p\\
    &\leq r^dB_r \omega_r^d (f,n^{-1}),
\end{split}
\end{equation}
where we set $\mathcal{T}_n^{0}(f)=f$.

For property (ii), we assume that $f(x_1,\ldots,x_{d-1},-x_d)=(-1)^{\eta_d}f(x_1,\ldots,x_{d-1},x_d)$. By \cref{Lp_poly_approx}, we have
\begin{equation}
    \mathcal{T}_{n}^d(f)(x)=\sum_{j=0}^{n}\alpha_{d,j} \int_{-\pi}^{\pi}f\left(x_1,\ldots,x_{d-1},\frac{u_d}{\pi}\right)\cos\left(ju_d-\frac{\eta_d \pi}{2}\right)du_d \cos{\left(j\pi x_d-\frac{\eta_d \pi}{2}\right)},
\end{equation}
with $\left|\alpha_{d,j}\right|\leq C_r n$.
Hence, by the definition of $\mathcal{T}_n^d$, we have
\begin{equation}
\begin{split}
    \mathcal{T}_n^d(f)&=\sum_{0\leq\boldsymbol{j}\leq n} \int_{[-\pi,\pi]^d}\left( f\left( \frac{\boldsymbol{u}}{\pi} \right) \prod_{k=1}^d \cos{\left(j_k u_k-\frac{\eta_k\pi}{2}\right)}\right) d\boldsymbol{u} \prod_{k=1}^d \alpha_{j_k,k} \cos\left(j_k\pi x_k-\frac{\eta_k\pi}{2}\right)\\
    &=\sum_{0\leq\boldsymbol{j}\leq n} a_{\boldsymbol{j}} \prod_{k=1}^d \cos(j_k\pi x_k-\frac{\eta_k\pi}{2}),
\end{split}
\end{equation}
where
\begin{equation}
    \left|a_{\boldsymbol{j}}\right|=\left| \prod_{k=1}^d \alpha_{j_k,k}\int_{[-\pi,\pi]^d}\left( f\left( \frac{\boldsymbol{u}}{\pi} \right) \prod_{k=1}^d \cos{\left(j_k u_k-\frac{\eta_k\pi}{2}\right)}\right) d\boldsymbol{u} \right|\leq (C_rn)^d \left\|f\right\|_1.
\end{equation}

For property (iii), by \cref{approx_tri}, for each $\cos{\left(j_k\pi x_k-\frac{\eta_k\pi}{2}\right)}$, there is a 3D ReLU NN $\Psi_{j_k,k}$ of width $8$, depth $N_2+1$ and height $$\max\left\{N_2+1+\lceil \log_2(N_2+1)+\frac{1}{2}\log_2 3 \rceil,\lceil \log_2 n \rceil\right\}$$ such that 
\begin{equation}
    \sup_{x\in [-1,1]}\left|\Psi_{j_k,k}(x)-\cos{\left(j_k\pi x_k-\frac{\eta_k\pi}{2}\right)}\right|\leq 2^{-N_2}.
\end{equation}
Let $\widetilde{\prod}_d:[-2,2]^d\to\mathbb{R}$ be defined in \cref{multi_prod}, which is implemented by a 3D ReLU NN of width $4+d$, depth $d-1$ and height $H^{\prime}+1$, such that
\begin{equation}
    \left\|\widetilde{\prod}_d\left(\Psi_{j_1,1}(x_1),\ldots,\Psi_{j_d,d}(x_1)\right)-\prod_{k=1}^d \Psi_{j_k,k}(x_k)\right\|_{p}\leq3(d-1)2^{-2H^{\prime}+2d-1}.
\end{equation}
We set
\begin{equation}
    \Phi(x)=\sum_{0\leq \boldsymbol{j}\leq n}a_{\boldsymbol{j}}\widetilde{\prod}_d\left( \Psi_{j_1,1}(x_1),\ldots,\Psi_{j_d,d}(x_d) \right),
\end{equation}
then it follows that
\begin{equation}
\begin{split}
\label{eq_f-Phi}
    \left\| f-\Phi\right\|_p
    \leq &\left\| f-\mathcal{T}_n^d(f) \right\|_p+\left\| \Phi-\mathcal{T}_n^d(f) \right\|_p\\
    \leq& r^d B_r \omega_r^d(f,n^{-1})_p+\\&(C_rn)^d \left\|f\right\|_1\sum_{0\leq\boldsymbol{j}\leq n}\left\|\widetilde{\prod}_d\left(\Psi_{j_1,1}(x_1),\ldots,\Psi_{j_d,d}(x_d)\right)-\prod_{k=1}^d\cos{\left( j_k \pi x_k -\frac{\eta_k\pi}{2} \right)}\right\|_p.
\end{split}
\end{equation}
To bound the summation in the RHS of Eq. \eqref{eq_f-Phi}, we have
\begin{equation}
\begin{split}
    &\left\|\widetilde{\prod}_d\left(\Psi_{j_1,1}(x_1),\ldots,\Psi_{j_d,d}(x_1)\right)-\prod_{k=1}^d\cos{\left( j_k \pi x_k -\frac{\eta_k\pi}{2} \right)}\right\|_p\\
    \leq& \left\|\widetilde{\prod}_d\left(\Psi_{j_1,1}(x_1),\ldots,\Psi_{j_d,d}(x_1)\right)-\prod_{k=1}^d\Psi_{j_k,k}(x_k)\right\|_p\\&+\left\|\prod_{k=1}^d\Psi_{j_k,k}(x_k)-\prod_{k=1}^d\cos{\left( j_k \pi x_k -\frac{\eta_k\pi}{2} \right)}\right\|_p\\
    \leq& 3(d-1)2^{-2H^{\prime}+2d-1}\\&+\sum_{k=1}^d\left\|\left(\Psi_{j_k,k}(x_k)-\cos{\left( j_k \pi x_k -\frac{\eta_k\pi}{2} \right)}\right)\prod_{l<k}\Psi_{j_l,l}(x_l)\prod_{k<l\leq d} \cos{\left( j_l \pi x_l -\frac{\eta_l\pi}{2} \right)}\right\|_p\\
    \leq & 3(d-1)2^{-2H^{\prime}+2d-1}+d\cdot2^{-N_2+2d}.
\end{split}
\end{equation}
Especially, taking $H^{\prime}=n=N_1$, we obtain
\begin{equation}
    \left\|f-\Phi\right\|_p\leq r^d B_r \omega_r^d\left(f,N_1^{-1}\right)_p+\frac{3d}{2} \left\|f\right\|_1 \left(4C_r N_1\right)^d 2^{-N_2}.
\end{equation}
Note that $\Phi$, by our construction, has a similar structure to \cref{Fig_HoloNet}, except that here we approximate the trigonometric basis instead of the Hermite polynomial basis. Hence, $\Phi$ can be given by a 3D ReLU NN of width $\max\left\{8N_1d,N_1^d(4+d)\right\}$, depth $N_2+d$, and height $\max\left\{ N_2+1+\lceil\log_2(N_2+1)+\frac{\log_2 3}{2}\rceil,
\right.$\\$\left.\lceil\log_2 N_1\rceil \right\}$.
\end{proof}

\begin{proof}[Proof of \cref{approx_Lp}]
    Note that for any function $f(x)$ defined on $[-1,1]$, we can always decompose $f$ into $f(x)=g(x)+h(x)$ such that $g$ is even and that $h$ is odd. Actually, we can simply set
\begin{equation}
    g(x)=\frac{f(x)+f(-x)}{2} \quad \text{and} \quad h(x)=\frac{f(x)-f(-x)}{2}.    
\end{equation} Then, we have
\begin{equation}
    \omega_r(g,t)_p\leq \omega_r(f,t)_p \quad \text{and}\quad\omega_r(h,t)_p\leq \omega_r(f,t)_p.
\end{equation}
Thus, for any $f\in L^p([-1,1])^d$, applying this decomposition to each variable in order yields $f=\sum_{i=1}^{2^d}f_i$, where each $f_i$ is either even or odd for each variable. Since $\mathcal{T}_n^d$ is linear, we have $\mathcal{T}_n^d(f) = \mathcal{T}_n^d\left(\sum_{i=1}^{2^d} f_i\right) = \sum_{i=1}^{2^d} \mathcal{T}_n^d(f_i)$. 
Then applying \cref{poly_approx_multivar} to each term $\mathcal{T}_n^d(f_i)$, the result follows directly.
\end{proof}

\noindent\textbf{Remark.}
    We can extend \cref{approx_Lp} to the weak-$L^1$ ($WL^1$) space, which contains all measurable functions on $[a,b]$ with bounded quasi-norm
\begin{equation}
\label{eq_weakL1}    \left\|f\right\|_{WL^1([a,b])}=\sup_{\lambda>0}\lambda\cdot m \left\{x\in[a,b]:|f(x)|\geq\lambda\right\},
\end{equation}
where $m(\cdot)$ is the Lebesgue measure. The corresponding weak-$L^1$ modulus of continuity is given by
\begin{equation}
\label{eq_weakL1Mod}
    \omega(f,t)_{WL^1}=\sup_{0<h\leq t}\left\|f(\cdot+h)-f(\cdot)\right\|_{WL^1([a,b-h])}.
\end{equation}
In \cite[Theorem 3.1]{Aliev2023Jackson}, the author constructed a trigonometric polynomial 
\begin{equation}
    U_{n}(x) = \frac{3}{2\pi n(2n^2 + 1)} \cdot \int_{-\pi}^{\pi} f_1(t) \cdot \left( \frac{\sin \frac{n(t-x)}{2}}{\sin \frac{t-x}{2}} \right)^4 dt
\end{equation}
of degree $2n-2$, such that
\begin{equation*}
    \left\|U_{n  }-f_1\right\|_{WL^1}\leq C^{\prime}\omega\left(f,\frac{\pi}{2n}\right)_{WL^1},
\end{equation*}
for some absolute constant $C^{\prime}>0$. Thus, similar to \cref{approx_Lp}, we can use \cref{eq_NNApproxTri} to approximate $f$ by a neural network $\mathcal{N}_{32(N_1-1),N_2+1,H}$ with $$H=\max\left\{N_2+1+\lceil \log_2(N_2+1)+\frac{1}{2}\log_2 3 \rceil,\lceil \log N_1 \rceil\right\},$$ whose function $\Phi$ satisfies that
\begin{equation*}
    \left\|\Phi-f\right\|_{WL^1}\leq C\left(\omega\left(f,N_1^{-1}\right)_{WL^1} +N_1^22^{-N_2}\right).
\end{equation*}
Note that the weak-$L^1$ space meets in many areas of mathematics, \textit{e.g.}, the conjugate functions of Lebesgue
 integrable functions. The difficulty of dealing with the weak-$L^1$ space is that the weak-$L^1$ space is not a normed space. Moreover, infinitely differentiable (even continuous) functions are not dense in this space. Our techniques can overcome the difficulty because we explicitly construct a trigonometric polynomial with the error controlled by the weak-$L^1$ modulus of continuity.

 \section{Conclusion}

In this work, we have addressed two fundamental questions concerning the approximation capabilities of ReLU NNs. Our work has revealed that the efficient representation of sawtooth functions is pivotal for constructing polynomial and trigonometric approximations. Leveraging this insight, we have utilized a height-augmented neural network architecture to give a more efficient representation of the sawtooth function. First, we have demonstrated that the height-augmented architecture yields superior approximation rates for several important classes of analytic functions. Compared to prior works, our construction can substantially improve the approximation efficiency. This provides a theoretically grounded pathway for designing more parameter-efficient networks. Second, for the first time, we have provided a quantitative and non-asymptotic approximation error bound of arbitrary order $r$ for general $L^p$ functions. This result offers explicit and computable error estimates, enriching the theoretical understanding of network approximation in the foundational spaces of modern analysis. We believe that our techniques can also be used to solve other important challenges in deep learning approximation theory in the future.
 \bibliographystyle{elsarticle-num} 
 \bibliography{reference}

%% else use the following coding to input the bibitems directly in the
%% TeX file.

%% Refer following link for more details about bibliography and citations.
%% https://en.wikibooks.org/wiki/LaTeX/Bibliography_Management

% \begin{thebibliography}{00}

% %% For numbered reference style
% %% \bibitem{label}
% %% Text of bibliographic item

% \bibitem{lamport94}
%   Leslie Lamport,
%   \textit{\LaTeX: a document preparation system},
%   Addison Wesley, Massachusetts,
%   2nd edition,
%   1994.

% \end{thebibliography}
\end{document}